\documentclass{article}
\usepackage[nonatbib, preprint]{neurips_2020}
\usepackage{multirow}
\input{style}
\theoremstyle{definition}
\newtheorem{theorem}{Theorem}[section]
\newtheorem{definition}[theorem]{Definition}

\newtheorem{proposition}[theorem]{Proposition}

\newtheorem{lemma}[theorem]{Lemma}
\newtheorem{remark}[theorem]{Remark}
\newtheorem{example}[theorem]{Example}

\newcommand{\be}{\begin{equation}}
\newcommand{\ee}{\end{equation}}
\newcommand{\bea}{\begin{equation*}\begin{aligned}}
\newcommand{\eea}{\end{aligned}\end{equation*}}
\newcommand{\ds}{\displaystyle}

\newcommand{\R}{\mathbb{R}}

\newcommand{\Max}{\max\limits_}
\newcommand{\Min}{\min\limits_}
\newcommand{\Sup}{\sup\limits_}
\newcommand{\Inf}{\inf\limits_}

\newcommand{\wh}{\widehat}
\newcommand{\mc}{\mathcal}
\newcommand{\mbb}{\mathbb}
\newcommand{\inner}[2]{\big \langle #1, #2 \big \rangle }

\DeclareMathOperator{\st}{s.t.}

\newcommand{\Let}{\triangleq}
\newcommand{\opt}{^\star}
\newcommand{\eps}{\varepsilon}

\newcommand{\EE}{\mathds{E}}

\newcommand{\PP}{\mbb P }
\newcommand{\QQ}{\mbb Q}

\newcommand{\dualvar}{\gamma}

\newcommand{\C}{\mc C}

\newcommand{\Pnom}{\wh \PP}

\newcommand{\KL}{\mathrm{KL}}

\newcommand{\dom}{\mathrm{dom}}
\newcommand{\inte}{\mathrm{int}} %interior

\usepackage{verbatim}
\usepackage[normalem]{ulem}
\title{Distributionally Robust Parametric \\ Maximum Likelihood Estimation}

\author{
	Viet Anh Nguyen \qquad \qquad  Xuhui Zhang \qquad \qquad Jos\'{e} Blanchet\\
	Stanford University, United States \\
	\texttt{ \{viet-anh.nguyen, xuhui.zhang, jose.blanchet\}@stanford.edu } 
	\AND
	Angelos Georghiou
     \\
	University of Cyprus, Cyprus \\
	\texttt{georghiou.angelos@ucy.ac.cy} 
}

\begin{document}
\maketitle

\begin{abstract}
We consider the parameter estimation problem of a probabilistic generative model prescribed using a natural exponential family of distributions. For this problem, the typical maximum likelihood estimator usually overfits under limited training sample size, is sensitive to noise and may perform poorly on downstream predictive tasks. To mitigate these issues, we propose a distributionally robust maximum likelihood estimator that minimizes the worst-case expected log-loss uniformly over a parametric Kullback-Leibler ball around a parametric nominal distribution. Leveraging the analytical expression of the Kullback-Leibler divergence between two distributions in the same natural exponential family, we show that the min-max estimation problem is tractable in a broad setting, including the robust training of generalized linear models. Our novel robust estimator also enjoys statistical consistency and delivers promising empirical results in both regression and classification tasks.
\end{abstract}

%%%%%%%%%%%%%%%%%%%%%%%%%%%%%%%%%%%
\section{Introduction} 
\label{sect:intro}

We are interested in the relationship between a response variable $Y$ and a covariate $X$ governed by the generative model
\be \label{eq:model}
Y | X = x \sim f \big(\cdot | \lambda(w_0, x) \big),
\ee
where $\lambda$ is a pre-determined function that maps the weight $w_0$ and the covariate $X$ to the parameter of the conditional distribution of $Y$ given $X$. The weight $w_{0}$ is unknown and is the main quantity of interest to be estimated. Throughout this paper, we assume that the distribution $f$ belongs to the exponential family of distributions. Given a ground measure $\nu$ on $\mc Y$, the exponential family is characterized by the density function
\[
f(y | \theta) = h(y) \exp\left( \inner{\theta}{T(y)} - \Psi(\theta)\right)
\]
with respect to $\nu$, where $\inner{\cdot}{\cdot}$ denotes the inner product, $\theta$ is the natural parameters, $\Psi$ is the log-partition function and $T$ is the sufficient statistics. The space of natural parameters is denoted by $\Theta = \left\{\theta:\int h(y)\exp(\inner{\theta}{T(y)})<\infty\right\} \subseteq \R^p$. We assume that the exponential family of distributions is regular, hence $\Theta$ is an open set, and $T_1(y),\ldots,T_p(y)$ are affinely independent~\cite[Chapter~8]{ref:barndorff2014introductory}. 

The generative setting~\eqref{eq:model} encapsulates numerous models which are suitable for regression and classification~\cite{ref:dobson2018introduction}. It ranges from logistic regression for classification~\cite{ref:hosmer2013applied}, Poisson counting regression~\cite{ref:hilbe2014poisson}, log-linear models~\cite{ref:christensen1990log} to numerous other generalized linear models~\cite{ref:dobson2018introduction}.

Given data $\{(\wh x_i, \wh y_i)\}_{i = 1, \ldots, N}$ which are assumed to be independently and identically distributed (i.i.d.) following the generative model~\eqref{eq:model}, we want to estimate the true value of $w_0$ that dictates~\eqref{eq:model}. If we use $\Pnom^{\text{emp}} = N^{-1} \sum_{i=1}^N \delta_{(\wh x_i, \wh y_i)}$ to denote the empirical distribution supported on the training data, and define $\ell_{\lambda}$ as the log-loss function with the parameter mapping $\lambda$
\be \label{eq:logloss-def}
\ell_\lambda(x, y, w) = \Psi( \lambda(w, x)) - \inner{T(y)}{\lambda(w, x)},
\ee
then the maximum likelihood estimation (MLE) produces an estimate $w_{MLE}$ by solving the following two equivalent optimization problems
%estimation (MLE) problem that searches for $w_{MLE}$ which best fits the data can be written as
\begin{subequations}
\begin{align} 
	 w_{MLE} &= \arg \Min{w \in \mc W} \textstyle\sum_{i=1}^{N} \frac{1}{N} \left( \Psi( \lambda(w, \wh x_i)) - \inner{T(\wh y_i)}{ \lambda(w, \wh x_i)} \right) \label{eq:MLE}\\
	 &= \arg \Min{w \in \mc W} ~\EE_{\Pnom^{\text{emp}}}[ \ell_\lambda(X, Y, w)]. \label{eq:MLE2}
\end{align}
\end{subequations} 
The popularity of MLE can be attributed to its consistency, asymptotic normality and efficiency~\cite[Section~5]{ref:vandervaart2000asymptotic}. Unfortunately, this estimator exhibits several drawbacks in the finite sample regime, or when the data carry high noise and may be corrupted. For example, the ML estimator for the Gaussian model recovers the sample mean, which is notoriously susceptible to outliers~\cite{ref:rousseeuw2011robust}. The MLE for multinomial logistic regression yields over-fitted models for small and medium sized data~\cite{ref:dejong2019sample}.

%The maximum likelihood estimator (MLE) is usually the method of choice for parameter training of these models thanks to is consistency, asymptotic normality and efficiency~\cite[Section~5]{ref:vandervaart2000asymptotic}. Unfortunately, the MLE exhibit several drawbacks in the finite sample regime, or when the data carries high noise and may be corrupted. For example, the MLE estimator for the Gaussian model recovers the sample mean, which is notoriously susceptible to outliers~\cite{ref:rousseeuw2011robust}. The MLE for multinomial logistic regression estimated by maximum likehood yields overfitted models for small and medium sized data~\cite{ref:dejong2019sample}.

Various strategies can be utilized to counter these adverse effects of the MLE in the limited data regime. The most common approach is to add a convex penalty  term such as a 1-norm or 2-norm of $w$ into the objective function of problem~\eqref{eq:MLE} to obtain different regularization effects, see \cite{ref:andrew2004feature, ref:su2006efficient} for regularized logistic regression. However, this approach relies on strong prior assumptions, such as the sparsity of $w_0$ for the 1-norm regularization, which may rarely hold in reality. Recently, dropout training has been used to prevent overfit and improve the generalization of the MLE~\cite{ref:srivastava2014dropout,ref:wager2013dropout,ref:wang2013fast}. Specific instances of dropout have been shown to be equivalent to a 2-norm regularization upon a suitable transformation of the inputs~\cite[Section~4]{ref:wager2013dropout}.  Another popular strategy to regularize problem~\eqref{eq:MLE} is by reweighting the samples instead of using a constant weight $1/N$ when calculating the loss. This approach is most popular in the name of weighted least-squares, which is a special instance of MLE problem under the Gaussian assumption with heteroscedastic noises.
 
Distributionally robust optimization (DRO) is an emerging scheme aiming to improve the out-of-sample performance of the statistical estimator, whereby the objective function of problem~\eqref{eq:MLE2} is minimized with respect to the most adverse distribution $\QQ$ in some ambiguity set. The DRO framework has produced many interesting regularization effects. If the ambiguity set is defined using the Kullback-Leibler (KL) divergence, then we can recover an adversarial reweighting scheme~\cite{li2019comparing,ref:bertsimas2018datadriven}, a variance regularization~\cite{ref:namkoong2016stochastic, ref:duchi2016statistics}, and adaptive gradient boosting~\cite{ref:blanchet2019a}. DRO models using KL divergence is also gaining recent attraction in many machine learning learning tasks~\cite{ref:faury2020distributionally, ref:si2020distributionally}. Another popular choice is the Wasserstein distance function which has been shown to have strong connections to regularization~\cite{ref:shafieezadeh2019regularization, ref:kuhn2019wasserstein}, and has been used in training robust logistic regression classifiers~\cite{ref:shafieezadeh2015distributionally, ref:blanchet2018optimal}. Alternatively, the robust statistics literature also consider the robustification of the MLE problem, for example, to estimate a robust location parameter~\cite{ref:huber1964robust}

Existing efforts using DRO typically ignore, or have serious difficulties in exploiting, the available information regarding the generative model~\eqref{eq:model}. While existing approaches using the Kullback-Leibler ball around the empirical distribution completely ignore the possibility of perturbing the conditional distribution, the Wasserstein approach faces the challenge of elicitating a sensible ground metric on the response variables. For a concrete example, if we consider the Poisson regression application, then $Y$ admits values in the space of natural numbers $\mbb N$, and deriving a global metric on $\mbb N$ that carries meaningful local information is nearly impossible because one unit of perturbation of an observation with $\wh y_i = 1$ does not carry the same amount of information as perturbing $\wh y_i = 1000$. The drawbacks of the existing methods behoove us to investigate a novel DRO approach that can incorporate the available information on the generative model in a systematic way.

\textbf{Contributions.} We propose the following \textit{distributionally robust MLE problem}
\be \label{eq:dro}
\Min{w \in \mc W} \Max{\QQ \in \mbb B(\Pnom)} \EE_{\QQ} \big[ \ell_\lambda(X, Y, w )\big],
\ee
which is a robustification of the MLE problem~\eqref{eq:MLE2} for generative models governed by an exponential family of distributions. The novelty in our approach can be summarized as follows.
\begin{itemize}[leftmargin = 5mm]
\item We advocate a new nominal distribution which is calibrated to reflect the available parametric information, and introduce a Kullback-Leibler ambiguity set that allows perturbations on both the marginal distribution of the covariate and the conditional distributions of the response.
\item We show that the min-max estimation problem~\eqref{eq:dro} can be reformulated as a single finite-dimensional minimization problem. Moreover, this reformulation is a convex optimization problem in broadly applicable settings, including the training of many generalized linear models.
\item We demonstrate that our approach can recover the adversarial reweighting scheme as a special case, and it is connected to the variance regularization surrogate. Further, we prove that our estimator is consistent and provide insights on the practical tuning of the parameters of the ambiguity set. We also shed light on the most adverse distribution in the ambiguity set that incurs the extremal loss for any estimate of the statistician. 
\end{itemize}

% \viet{Xuhui, please check all the dimensions:
% $\Theta \in \R^p$, $\mc X \in \R^n$, $\mc Y \in \R^m$, $\mc W$ is finite dimensional space, checked}

\textbf{Technical notations.} The variables $(X, Y)$ admit values in $\mc X \times \mc Y \subseteq \R^n \times \R^m$, and $\mc W$ is a finite-dimensional set. The mapping $\lambda: \mc W \times \mc X \to \Theta \subseteq \R^p$ is jointly continuous, and $\inner{\cdot}{\cdot}$ denotes the inner product in $\R^p$. For any set $\mc S$, $\mc M(\mc S)$ is the space of all probability measures with support on $\mc S$. We use $\xrightarrow{p.}$ to denote convergence in probability, and $\xrightarrow{d.}$ to denote convergence in distribution.  All proofs are relegated to the appendix.

%%%%%%%%%%%%%%%%%
\section{Distributionally Robust Estimation with a Parametric Ambiguity Set} \label{sect:ambiguity-set} 

We delineate in this section the ingredients of our distributionally robust MLE using parametric ambiguity set. Since the log-loss function is pre-determined, we focus solely on eliciting a nominal probability measure and the neighborhood surrounding it, which will serve as the ambiguity set.

While the typical empirical measure $\Pnom^{\text{emp}}$ may appear at first  as an attractive option for the nominal measure, $\Pnom^{\text{emp}}$ does not reflect the parametric nature of the conditional measure of $Y$ given $X$. Consequently, to robustify the MLE model, we need a novel construction of the nominal distribution~ $\Pnom$. 

Before proceeding, we assume w.l.o.g. that the dataset $\{ (\wh x_i, \wh y_i)\}_{i=1, \ldots, N}$  consists of $C \le N$ distinct observations of $X$, each value is denoted by $\wh x_c$ for $c = 1, \ldots, C$, and the number of observations with the same covariate value $\wh x_c$ is denoted by $N_c$. This regrouping of the  data by $\wh x_c$ typically enhances the statistical power of estimating the distribution conditional on the event $X = \wh x_c$.

We posit the following \emph{parametric nominal distribution} $\Pnom \in \mc M(\mc X \times \mc Y)$. This
distribution is fully characterized by  $(p+1)C$ parameters: a probability vector $\wh p\in\mathbb{R}_+^C$ whose elements sum up to 1 and a vector of nominal natural parameters $\wh \theta \in \Theta^C \subseteq (\R^p)^C$. Mathematically, $\Pnom$ satisfies
\be \label{eq:nominal}
\left\{ \begin{array}{l}
	\Pnom (\{\wh x_c\} \times A) = \Pnom_X(\{\wh x_c\}) \Pnom_{Y|\wh x_c}(A) \qquad \forall \wh x_c, \forall A \subseteq \mc Y \text{ measurable} \\
	\Pnom_X = \sum_{c=1}^C \wh p_c \delta_{\wh x_c}, \quad \Pnom_{Y|\wh x_c} \sim f(\cdot | \wh \theta_c) \; \forall c.
\end{array}
\right.
\ee
The first equation  indicates that the nominal measure $\Pnom$ can be decomposed into a marginal distribution of the covariates $X$ and a collection of conditional measures of $Y$ given $X$ using the definition of the conditional probability measure ~\cite[Theorem~9.2.2]{ref:stroock2011probability}. The second line  stipulates that the nominal marginal distribution $\Pnom_X$ of the covariates is a discrete distribution supported on $\wh x_c$, $c = 1, \ldots, C$. Moreover, for each $c$, the nominal conditional distribution of $Y$ given $X = \wh x_c$ is a distribution in the exponential family with parameter $\wh \theta_c$. Notice that the form of~$\Pnom$ in~\eqref{eq:nominal} is chosen to facilitate the injection of parametric information $\wh \theta_c$ into the nominal distribution, and it is also necessary to tie $\Pnom$ to the MLE problem using the following notion of MLE-compatibility.
\begin{comment}
We posit the following parametric nominal distribution $\Pnom \in \mc M(\mc X \times \mc Y)$ which can be written as
\be \label{eq:nominal}
\left\{ \begin{array}{l}
	\Pnom (\{\wh x_c\} \times A) = \Pnom_X(\{\wh x_c\}) \Pnom_{Y|\wh x_c}(A) \qquad \forall \wh x_c, \forall A \subseteq \mc Y \text{ measurable} \\
	\Pnom_X = \sum_{c=1}^C \wh p_c \delta_{\wh x_c}, \quad \Pnom_{Y|\wh x_c} \sim f(\cdot | \wh \theta_c) \; \forall c,
\end{array}
\right.
\ee
where $\wh p$ is a $C$-dimensional probability vector containing non-zero elements summing up to 1, and $\wh \theta_c \in \Theta$ for all $c$. The first equation  indicates that the nominal measure $\Pnom$ can be decomposed into a marginal distribution of the covariates $X$ and a collection of conditional measures of $Y$ given $X$ using the definition of the conditional probability measure ~\cite[Theorem~9.2.2]{ref:stroock2011probability}. The second line  stipulates that the nominal marginal distribution $\Pnom_X$ of the covariates is a discrete distribution supported on $\wh x_c$, $c = 1, \ldots, C$. Moreover, for each $c$, the nominal conditional distribution of $Y$ given $X = \wh x_c$ is a distribution in the exponential family with parameter $\wh \theta_c$. Notice that the form of~$\Pnom$ in~\eqref{eq:nominal} is chosen to facilitate the injection of parametric information $\wh \theta_c$ into the nominal distribution, and it is also necessary to tie $\Pnom$ to the MLE problem using the following notion of MLE-compatibility.
\end{comment}
\begin{definition}[MLE-compatible nominal distribution] \label{def:mle-compatible}
	A nominal distribution $\Pnom$ of the form~\eqref{eq:nominal} is MLE-compatible with respect to the log-loss function $\ell_\lambda$ if the optimal solution $\wh w = \arg \min_{w \in \mc W}~\EE_{\Pnom}[ \ell_{\lambda}(X, Y, w)]$ coincides with the estimator $w_{MLE}$ that solves~\eqref{eq:MLE}.
\end{definition}
Definition~\ref{def:mle-compatible} indicates that $\Pnom$ is compatible for the MLE problem if the MLE solution $w_{MLE}$ is recovered by solving problem~\eqref{eq:MLE2} where the expectation is now taken under $\Pnom$. Therefore, MLE-compatibility implies that $\Pnom$ and $\Pnom^{\text{emp}}$ are equivalent in the MLE problem.%\xuhui{perhaps the last sentence can be ``therefore, MLE-compatibility implies that $\Pnom$ and $\Pnom^{\text{emp}}$ are in a sense equivalent'' to stress that the definition of euivalent is new.}

The next examples suggest two possible ways of calibrating an MLE-compatible $\Pnom$ of the form~\eqref{eq:nominal}.
\begin{example}[Compatible nominal distribution I] \label{ex:nominal-1}
	If $\Pnom$ is chosen of the form~\eqref{eq:nominal} with $\wh p_c = N_c/N$ and
	$\wh \theta_c = (\nabla \Psi)^{-1}\big((N_c)^{-1} \sum_{\wh x_i = \wh x_c} T(\wh y_i) \big)\in \Theta$ for all $c$,	then $\Pnom$ is MLE-compatible.
\end{example}

\begin{example}[Compatible nominal distribution II] \label{ex:nominal-2}
	If $\Pnom$ is chosen of the form~\eqref{eq:nominal} with $\wh p_c = N_c/N$ and $\wh \theta_c = \lambda(w_{MLE}, \wh x_c)$ for all $c$, where $w_{MLE}$ solves~\eqref{eq:MLE}, then $\Pnom$ is MLE-compatible.
\end{example}

We now detail the choice of the dissimilarity measure which is used to construct the neighborhood surrounding the nominal measure $\Pnom$. For this, we will  use the Kullback-Leiber divergence.

\begin{definition}[Kullback-Leibler divergence] \label{def:KL}
	Suppose that $\PP_1$ is absolutely continuous with respect to $\PP_2$, the Kullback-Leibler (KL) divergence from $\PP_1$ to $\PP_2$ is defined as $\KL(\PP_1 \parallel \PP_2) \Let \EE_{\PP_1} \left[ \log (\mathrm{d} \PP_1/\mathrm{d} \PP_2) \right]$,
	where $\mathrm{d} \PP_1/\mathrm{d} \PP_2$ is the Radon-Nikodym derivative of $\PP_1$ with respect to $\PP_2$.
\end{definition}

% \subsection{Method 1}
% The ambiguity set for the covariate $X$ is defined as
% \[
% 	\mbb B_X \Let \left\{ \QQ_{X} \in \mc M (\mc X): \KL(\QQ_X \parallel \Pnom_{X} ) \leq \eps \right\}
% \]
% for some radius $\eps \ge 0$.

% Notice that we have only 1 ambiguity set $\mbb B_X$, however, we have $C$ ambiguity sets $\mbb B_{Y|\wh x_c}$, one for each value of $c$. Each ambiguity set $\mbb B_{Y|\wh x_c}$ differ by the center and also by the radius $\rho_c$. Since $\rho_c$ represents the level of ambiguity with respect to the conditional probability given the covariate $\wh x_c$, it is natural to put an individual radius for each ambiguity set. This will create $C$ additional parameters to tune, however, as will be described in the the numerical section, we can combine all these $\rho_c$ down to tuning only one parameter.

% The ambiguity set $\mbb B_{X}$ will tackle the population composition, while the ambiguity sets $\mbb B_{Y|X}$ will tackle the statiscal error in the estimation of the class-dependent distributions.

% \be \label{eq:B-def-1}
% \mbb B(\Pnom) \Let \left\{ \QQ\in \mc M(\mc X \times \mc Y): 
% \begin{array}{l}
% \exists \QQ_X \in \mc M(\mc X),~\QQ_{Y|\wh x_c} \in \mc M(\mc Y),~\theta_c \in \Theta \quad \forall c \\
% \QQ_{Y| \wh x_c} \sim f(\cdot | \theta_c),~
% \QQ (\{\wh x_c\} \times A) = \QQ_X(\{\wh x_c\}) \QQ_{Y|\wh x_c}(A) ~ \forall A \in \mc F(\mc Y), \quad \forall c \\ 
% \QQ_X \in \mbb B_X,~\QQ_{Y|\wh x_c} \in \mbb B_{Y|\wh x_c} \quad \forall c
% \end{array}
% \right\}
% \ee

% \subsection{Method 2}

The KL divergence is an ideal choice in our setting for numerous reasons. Previously, DRO problems with a KL ambiguity set often result in tractable finite-dimensional reformulations~\cite{ref:bental2013robust,  ref:hu2013kullback, ref:bertsimas2018datadriven}. More importantly, the manifold of exponential family of distributions equipped with the KL divergence inherits a natural geometry endowed by a dually flat and invariant Riemannian structure~\cite[Chapter~2]{ref:amari2016information}. Furthermore, the KL divergence between two distributions in the same exponential family  admits a closed form expression~\cite{ref:banerjee2005clustering, ref:amari2016information}.
\begin{lemma}[{KL divergence between distributions from exponential family}]  \label{lemma:KL-exp-main} The KL divergence from $\QQ_1\!\sim\!f(\cdot | \theta_1)$ to $\QQ_2\!\sim\!f(\cdot|\theta_2)$ amounts to
	$\KL (\QQ_1 \parallel \QQ_2)\!=\!\inner{\theta_1 - \theta_2}{\nabla \Psi(\theta_1)} - \Psi(\theta_1) + \Psi(\theta_2).
	$
\end{lemma}

Using the above components, we are now ready to introduce our ambiguity set $\mbb B(\Pnom)$ as
\be \label{eq:B-def}
	\mbb B(\Pnom)\!\Let\!\left\{ \QQ\!\in\!\mc M(\mc X\!\times\!\mc Y): \!\!\!
	\begin{array}{l}
		\exists \QQ_X \in \mc M(\mc X),~\exists\theta_c \in \Theta \text{ such that } \QQ_{Y| \wh x_c} \sim f(\cdot | \theta_c) \quad \forall c \\
		\QQ (\{\wh x_c\} \times A)\!=\!\QQ_X(\{\wh x_c\}) \QQ_{Y|\wh x_c}(A) ~~ \forall c, \forall A \subseteq \mc Y~\text{measurable} \\ 
		\KL(\QQ_{Y|\wh x_c} \parallel \Pnom_{Y|\wh x_c}) \leq \rho_c \quad \forall c\\
		\KL(\QQ_X \parallel \Pnom_X) + \EE_{\QQ_X}[\sum_{c=1}^C \rho_c \mathbbm{1}_{\wh x_c}(X)] \le \eps
	\end{array}
	\right\}
\ee
parametrized by a marginal radius $\eps$ and a collection of the conditional radii $\rho_c$. Any distribution $\QQ \in \mbb B(\Pnom)$ can be decomposed into a marginal distribution  $\QQ_X$ of the covariate and an ensemble of parametric conditional distributions $\QQ_{Y|\wh x_c} \sim f(\cdot | \theta_c)$ at every event $X = \wh x_c$. The first inequality in~\eqref{eq:B-def} restricts the parametric conditional distribution $\QQ_{Y|\wh x_c}$ to be in the $\rho_c$-neighborhood from the nominal $\Pnom_{Y|\wh x_c}$ prescribed using the KL divergence, while the second inequality imposes a similar restriction for the marginal distribution $\QQ_X$. One can show that for any conditional radii $\rho \in \R_+^C$ satisfying $\sum_{c=1}^C \wh p_c \rho_c \leq \eps$, $\mbb B(\Pnom)$ is non-empty with $\Pnom\!\in\!\mbb B(\Pnom)$. Moreover, if all $\rho$ and $\eps$ are zero, then $\mbb B(\Pnom)$ becomes the singleton set $\{\Pnom\}$ that contains only the nominal distribution. 

The set $\mbb B(\Pnom)$ is a \textit{parametric ambiguity set}: all conditional distributions $\QQ_{Y|\wh x_c}$ belong to the same parametric exponential family, and at the same time, the marginal distribution $\QQ_X$ is absolutely continuous with respect to a discrete distribution $\Pnom_X$ and hence $\QQ_X$ can be parametrized using a $C$-dimensional probability vector.

At first glance, the ambiguity set $\mbb B(\Pnom)$ looks intricate and one may wonder whether the complexity of $\mbb B(\Pnom)$ is necessary.  In fact, it is appealing to consider the ambiguity set
\be \label{eq:B}
	\mc B(\Pnom)\!\Let\!\left\{ \QQ\!\in\!\mc M(\mc X\!\times\!\mc Y): \!\!\!
	\begin{array}{l}
	\exists \QQ_X \in \mc M(\mc X),~\exists\theta_c \in \Theta \text{ such that } \QQ_{Y| \wh x_c} \sim f(\cdot | \theta_c) \quad \forall c \\
	\QQ (\{\wh x_c\} \times A) = \QQ_X(\{\wh x_c\}) \QQ_{Y|\wh x_c}(A)~\forall c, \forall A \subseteq \mc Y~\text{measurable}\\
	\KL(\QQ\parallel \Pnom) \le \eps
	\end{array}
	\right\}
\ee
which still preserves the parametric conditional structure and entails only one KL divergence constraint on the \textit{joint} distribution space.
Unfortunately, the ambiguity set $\mc B(\Pnom)$ may be overly conservative as pointed out in the following result.

\begin{proposition} \label{prop:conservative}
	Denote momentarily the ambiguity sets~\eqref{eq:B-def} and~\eqref{eq:B} by $\mbb B_{\eps, \rho}(\Pnom)$ and $\mc B_{\eps}(\Pnom)$ to make the dependence on the radii explicit. For any nominal distribution $\Pnom$ of the form~\eqref{eq:nominal} and any radius $\eps \in \R_{+}$, we have
	\[
	   \textstyle \mc B_\eps(\Pnom) = \bigcup_{\rho \in \R_+^C} \mbb B_{\eps, \rho}(\Pnom).
	\] 
	%{\color{red}For any conditional radii $\rho \in \R_+^C$ satisfying $\sum_{c=1}^C \wh p_c \rho_c \leq \eps$, $\mbb B_{\eps, \rho}(\Pnom)$ is non-empty with $\Pnom\!\in\!\mbb B_{\eps, \rho}(\Pnom)$.}
\end{proposition}

Proposition~\ref{prop:conservative} suggests that the ambiguity set $\mc B(\Pnom)$ can be significantly bigger than $\mbb B(\Pnom)$, and that the solution of the distributionally robust MLE problem~\eqref{eq:dro} with $\mbb B(\Pnom)$ being replaced by $\mc B(\Pnom)$ is potentially too conservative and may lead to undesirable or uninformative results.

The ambiguity set $\mbb B(\Pnom)$ requires $1+C$  parameters, including one marginal radius $\eps$ and $C$ conditional radii $\rho_c$, $c= 1,\ldots, C$, which may be cumbersome to tune in the implementation. Fortunately, by the asymptotic result in Lemma~\ref{lemma:asymptotic-joint}, the set of radii $\rho_c$ can be tuned simultaneously using the same scaling rate, which will significantly reduce the computational efforts for parameter tuning.

%%%%%%%%%%%%%%%%%%%%%%%%%%%%%%
\section{Tractable Reformulation}
\label{sect:refor}

We devote this section to study the solution method for the min-max problem~\eqref{eq:dro} by transforming it into a finite dimensional minimization problem. To facilitate the exposition, we denote the ambiguity set for the conditional distribution of $Y$ given $X = \wh x_c$ as
\be \label{eq:cond-ambi}
	\mbb B_{Y|\wh x_c} \Let \left\{ \QQ_{Y|\wh x_c} \in \mc M(\mc Y): \exists \theta \in \Theta,~\QQ_{Y|\wh x_c}(\cdot) \sim f(\cdot | \theta),~\KL(\QQ_{Y|\wh x_c}  \parallel  \Pnom_{Y|\wh x_c}) \leq \rho_c  \right\}.
\ee
As a starting point, we first show the following decomposition of the worst-case expected loss under the ambiguity set $\mbb B(\Pnom)$ for any measurable loss function.

\begin{proposition}[Worst-case expected loss] \label{prop:refor}
	Suppose that~$\mbb B(\Pnom)$ is defined as in~\eqref{eq:B-def} for some $\eps \in \R_+$ and $\rho \in \R_+^C$ such that $\sum_{c=1}^C \wh p_c \rho_c \le \eps$. For any function $L: \mc X \times \mc Y \to \R$ measurable, we have
	\[
	\Sup{\QQ \in \mbb B(\Pnom)} \EE_{\QQ}\left[  L(X, Y) \right]  = \left\{
	\begin{array}{cl}
		\inf & \alpha + \beta \eps + \beta \ds \textstyle\sum_{c=1}^C \wh p_c \exp\left( \beta^{-1}(t_c - \alpha) - \rho_c - 1\right) \\
		\st & t \in \R^C,\; \alpha \in \R,\;\beta \in \R_{++}  \\
		& \Sup{\QQ_{Y|\wh x_c} \in \mbb B_{Y|\wh x_c}}  \EE_{\QQ_{Y|\wh x_c}} \left[ L(\wh x_c, Y)\right] \leq t_c \quad \forall c =1, \ldots, C.
	\end{array}
	\right.
	\]
\end{proposition}

Proposition~\ref{prop:refor} leverages the decomposition structure of the ambiguity set $
\mbb B(\Pnom)$ to reformulate the worst-case expected loss into an infimum problem that involves $C$ constraints, where each constraint is a hypergraph reformulation of a worst-case conditional expected loss under the ambiguity set $\mbb B_{Y|\wh x_c}$. Proposition~\ref{prop:refor} suggests that to reformulate the min-max estimation problem~\eqref{eq:dro}, it suffices now to reformulate the worst-case conditional  expected log-loss
\be \label{eq:inner}
\Sup{\QQ_{Y|\wh x_c} \in \mbb B_{Y|\wh x_c}}  \EE_{\QQ_{Y|\wh x_c}} \left[\ell_\lambda(\wh x_c, Y, w ) \right]
\ee
for each value of $\wh x_c$ into a dual infimum problem. Using Lemma~\ref{lemma:KL-exp-main}, one can rewrite $\mbb B_{Y|\wh x_c}$ in \eqref{eq:cond-ambi} using the natural parameter representation as 
\[
	\mbb B_{Y|\wh x_c}\!=\!\left\{ \QQ_{Y|\wh x_c} \!\in\!\mc M(\mc Y)\!:\!\exists \theta \in \Theta,\QQ_{Y|\wh x_c}(\cdot)\!\sim\!f(\cdot | \theta),\inner{\theta - \wh \theta_c}{\nabla \Psi(\theta)} - \Psi(\theta) + \Psi(\wh \theta_c) \leq \rho_c  \right\}.
\]
Since $\Psi$ is convex~\cite[Lemma~1]{ref:banerjee2005clustering}, it is possible that $\mbb B_{Y|\wh x_c}$ is represented by a non-convex set of natural parameters and hence reformulating~\eqref{eq:inner} is non-trivial. Surprisingly, the next proposition asserts that problem~\eqref{eq:inner} always admits a convex reformulation.
\begin{proposition}[Worst-case conditional expected log-loss] \label{prop:conditional-refor}
	For any $\wh x_c \in \mc X$ and $w \in \mc W$, the worst-case conditional expected log-loss~\eqref{eq:inner} is equivalent to the univariate convex optimization problem
	\be \label{eq:inner-refor}
    \Inf{\dualvar_c \in \R_{++}}~\dualvar_c \big(\rho_c - \Psi(\wh \theta_c) \big) + \dualvar_c \Psi\big(\wh \theta_c - \dualvar_c^{-1}\lambda(w, \wh x_c) \big) + \Psi\big(\lambda(w, \wh x_c)\big).
	\ee
\end{proposition}

A reformulation for the worst-case conditional expected log-loss was proposed in~\cite{ref:hu2013kullback}. Nevertheless, the results in~\cite[Section~5.3]{ref:hu2013kullback} requires that the sufficient statistics $T(y)$ is a linear function of $y$. The reformulation~\eqref{eq:inner-refor}, on the other hand, is applicable when $T$ is a \textit{non}linear function of $y$. Examples of exponential family of distributions with nonlinear $T$ are (multivariate) Gaussian, Gamma and Beta distributions. The results from Propositions~\ref{prop:refor} and \ref{prop:conditional-refor} lead to the reformulation of the distributionally robust estimation problem~\eqref{eq:dro}, which is the main result of this section.

\begin{theorem}[Distributionally robust MLE reformulation] \label{thm:main}
	The distributionally robust MLE problem~\eqref{eq:dro} is tantamount to the following finite dimensional optimization problem
	\be \label{eq:refor}
	\begin{array}{cll}
		\inf & \alpha + \beta \eps + \beta \ds \textstyle \sum_{c=1}^C \wh p_c \exp( \beta^{-1}(t_c - \alpha) - \rho_c - 1 ) \\
		\st & w \in \mc W,\; \alpha \in \R,\;\beta \in \R_{++},\;\gamma\in \R_{++}^C,\; t \in \R^C  \\[1ex]
		& \dualvar_c \big(\rho_c - \Psi(\wh \theta_c) \big) + \dualvar_c \Psi\big( \wh \theta_c - \dualvar_c^{-1}\lambda(w, \wh x_c) \big) + \Psi \big(\lambda(w, \wh x_c) \big) \leq t_c & \forall c =1, \ldots, C.
	\end{array}
	\ee
	In generalized linear models with $\lambda: (w, x) \mapsto w^\top x$ and $\mc W$ being convex, problem~\eqref{eq:refor} is convex.
\end{theorem}

Below we show how the Poisson and logistic regression models fit within this framework.

\begin{example}[Poisson counting model]
\label{ex:poisson}
The Poisson counting model with the ground measure $\nu$ being a counting measure on $\mc Y = \mbb N$, the sufficient statistic $T(y)=y$, the natural parameter space $\Theta = \R$ and the log-partition function $\Psi(\theta) = \exp(\theta)$. If $\lambda (w, x) = w^\top x$, we have
\[
	Y | X = x \sim \mathrm{Poisson} \big(w_0^\top x \big), \qquad \PP(Y = k | X = x) = (k!)^{-1}\exp(kw_0^\top x-e^{w_0^\top x}).
\]
The distributionally robust MLE is equivalent to the following convex optimization problem
	\be \label{eq:Poisson:refo:exp}
	\begin{array}{cll}
		\inf & \alpha + \beta \eps  + \beta \ds \textstyle\sum_{c =1}^C \wh p_c \exp\left( \beta^{-1}(t_c  - \alpha) - \rho_c - 1 \right)\\
		\st & w \in \mc W,\; \alpha \in \R,\; \beta \in \R_{++},\;\dualvar \in \R_{++}^C, \; t \in \R^C \\[1ex]
		& \dualvar_c \big( \rho_c - \exp( \wh \theta_c) \big) + \dualvar_c \exp\big(\wh \theta_c -w^\top \wh x_c/\dualvar_c \big) + \exp\big(w^\top \wh x_c\big) \le t_c & \forall c = 1, \ldots, C.
	\end{array}
	\ee
\end{example}

\begin{example}[Logistic regression]
The logistic regression model is specified with $\nu$ being a counting measure on $\mc Y = \{0, 1\}$, the sufficient statistic $T(y)=y$, the natural parameter space $\Theta = \R$ and the log-partition function $\Psi(\theta) = \log\big(1+ \exp(\theta)\big)$. If $\lambda(w, x) = w^\top x$, we have
\[
	Y | X = x \sim \mathrm{Bernoulli} \big((1+\exp(-w_0^\top x))^{-1} \big), \qquad \PP(Y = 1 | X = x) = (1+\exp(-w_0^\top x))^{-1}.
\]
The distributionally robust MLE is equivalent to the following convex optimization problem
\begin{equation}\label{eq:Logistic:refo:exp}
	\begin{array}{cl}
		\inf & \alpha + \beta \eps  + \beta \ds \textstyle\sum_{c =1}^C \wh p_c \exp\left(\beta^{-1} (t_c  - \alpha) - \rho_c - 1 \right)\\
		\st & w \in \mc W,\; \alpha \in \R,\; \beta \in \R_{++},\;\dualvar \in \R_{++}^C, \; t \in \R^C \\[1ex]
		& \dualvar_c \big( \rho_c \!-\! \log(1\!+\!\exp(\wh \theta_c)) \big) \!+\!\dualvar_c \log\big(1\!+\!\exp(\wh \theta_c \!-\! w^\top \wh x_c/\dualvar_c) \big)\!+ \!\log\big(1\!+\!\exp(w^\top \wh x_c)\big)\!\le\!t_c ~ \forall c.
	\end{array}
\end{equation}
\end{example}
Problems~\eqref{eq:Logistic:refo:exp} and~\eqref{eq:Poisson:refo:exp} can be solved by exponential conic solvers such as ECOS~\cite{domahidi2013ecos} and MOSEK~\cite{mosek}.

%%%%%%%%%%%%%%%%%%%%%%%%%%%%%%%%%%%%%%%%%%%%

\section{Theoretical Analysis}
\label{sect:analysis}

In this section, we provide an in-depth theoretical analysis of our estimator. We first show that our proposed estimator is tightly connected to several existing regularization schemes.

\begin{proposition}[Connection to the adversarial reweighting scheme] \label{prop:reweighting}
    Suppose that $\wh x_i$ are distinct and $\rho_c = 0$ for any $c = 1, \ldots, N$. If $\Pnom$ is of the form~\eqref{eq:nominal} and chosen according to Example~\ref{ex:nominal-1}, then the distributionally robust estimation problem~\eqref{eq:dro} is equivalent to
     \be \notag
 	\Min{w \in \mc W}~\Sup{\QQ: \KL(\QQ \parallel \Pnom^{\mathrm{emp}}) \le \eps} \EE_{\QQ}[ \ell_\lambda(X, Y, w)].
    \ee
\end{proposition}
Proposition~\ref{prop:reweighting} asserts that by setting the conditional radii to zero, we can recover the robust estimation problem where the ambiguity set is a KL ball around the empirical distribution $\Pnom^{\text{emp}}$, which has been shown to produce the adversarial reweighting effects~\cite{li2019comparing,ref:bertsimas2018datadriven}. Recently, it has been shown that distributionally robust optimization using $f$-divergences is statistically related to the variance regularization of the empirical risk minimization problem~\cite{ref:namkoong2017variance}. Our proposed estimator also admits a variance regularization surrogate, as asserted by the following proposition.
\begin{proposition}[Variance regularization surrogate]
\label{prop:surrogate}
    Suppose that $\Psi$ has locally Lipschitz continuous gradients. For any fixed $\wh \theta_c \in \Theta$, $c=1,\ldots,C$, there exists a constant $m>0$ that depends only on $\Psi$ and $\wh \theta_c$, $c=1, \ldots, C$, such that for any $w \in \mc W$ and $\eps \ge \sum_{c = 1}^C \wh p_c \rho_c$, we have 
    \[
        \sup_{\QQ \in \mbb B(\Pnom)} \EE_{\QQ}[\ell_\lambda(X, Y, w)] \le \EE_{\Pnom}[\ell_\lambda(X, Y, w)] + \kappa_1\sqrt{\mathrm{Var}_{\Pnom}\left(\ell_\lambda(X, Y, w ) \right)} + \kappa_2  \| \lambda(w, \wh x_c) \|_2,
    \]
    where  $\kappa_1 = \sqrt{2\eps}/(\min_{c}\sqrt{\wh p_c})$ and $\kappa_2= \sqrt{2\max_{c}\rho_c/m}$.
\end{proposition}
One can further show that for sufficiently small $\rho_c$, the value of $m$ is proportional to the inverse of the local Lipschitz constant of $\nabla \Psi$ at $\wh \theta_c$, in which case $\kappa_2$ admits an explicit expression (see Appendix~\ref{sec:app-aux}). Next, we show that our robust estimator is also consistent, which is a highly desirable statistical property.
\begin{theorem}[Consistency] \label{thm:consistency}
    Assume that $w_0$ is the unique solution of the problem $\min_{w\in\mc W}\EE_{\PP}\left[\ell_{\lambda}(X,Y,w)\right]$, where $\PP$ denotes the true distribution. Assume that $\mathcal{X}$ has finite cardinality, $\Theta=\mathbb{R}^p$, $\Psi$ has locally Lipschitz continuous gradients, and $\ell_\lambda(x,y,w)$ is convex in $w$ for each $x$ and $y$. If $\wh \theta_c \xrightarrow{p.} \lambda(w_0, \wh x_c)$ for each $c$, $\eps\to0,\rho_c\to0$ and $\eps\geq \sum_{c=1}^C\wh p_c\rho_c$ with probability going to $1$, then the distributionally robust estimator $w\opt$ that solves~\eqref{eq:dro} exists with probability going to $1$, and $w\opt\xrightarrow{p.} w_0$.
\end{theorem}
One can verify that choosing $\wh \theta_c$ using Examples~\ref{ex:nominal-1} and~\ref{ex:nominal-2} will satisfy the condition $\wh \theta_c \xrightarrow{p.} \lambda(w_0, \wh x_c)$, and as a direct consequence, choosing $\Pnom$ following these two examples will result in a consistent estimator under the conditions of Theorem~\ref{thm:consistency}. 

We now consider the asymptotic scaling rate of $\rho_c$ as the number $N_c$ of samples with the same covariate $\wh x_c$ tends to infinity. Lemma~\ref{lemma:asymptotic-joint} below asserts that $\rho_c$ should scale at the rate $N_c^{-1}$. Based on this result, we can set $\rho_c = a N_c^{-1}$ for all $c$, where $a > 0$ is a tuning parameter. This reduces significantly the burden of tuning $\rho_c$ down to tuning a single parameter $a$.

\begin{lemma}[Joint asymptotic convergence] \label{lemma:asymptotic-joint}
   Suppose that $| \mathcal X | = C$ with $\mathbb P(X = \widehat x_c) > 0$. Let $\theta_c = \lambda(w_0, \widehat x_c)$ and $\widehat{\mathbb P}$ be defined as in Example~2.2. Let $V_c=D_c  \mathrm{Cov}_{f(\cdot|\theta_c)}(T(Y))D_c^\top$, where
   $
    D_c= J (\nabla \Psi)^{-1}( \mathds E_{f(\cdot|\theta_c)}[T(Y)])
$ and $J$ denotes the Jacobian operator.
Then the following joint convergence holds 
    \vspace{-1.3mm}
    \begin{equation}\label{eq:jointcon}
    \big(N_1\times\mathrm{KL}(f(\cdot |\theta_1)\parallel f(\cdot |\widehat\theta_1)), \ldots, N_C\times \mathrm{KL}(f(\cdot|\theta_C)\parallel f(\cdot|\widehat\theta_C))\big)^\top \xrightarrow{d.} Z \qquad \text{as}\qquad N\to\infty,
    \end{equation}
    where $Z = (Z_1,\ldots, Z_C)^\top$ with $Z_c = \frac{1}{2}R_c^\top\nabla^2\Psi(\theta_c)R_c$, $R_c$ are independent and $R_c\sim\mathcal N(0,V_c)$.
\end{lemma}

Assuming $w_{MLE}$ that solves~\eqref{eq:MLE} is asymptotically normal with square-root convergence rate, we remark that the asymptotic joint convergence~\eqref{eq:jointcon} also holds for $\Pnom$ in Example~\ref{ex:nominal-2}, though in this case the limiting distribution $Z$ takes a more complex form that can be obtained by the delta method.

% By delta method, the condition~\eqref{eq:jointcon} also holds for $\Pnom$ in Example~\ref{ex:nominal-2}, assuming $w_{MLE}$ that solves~\eqref{eq:MLE} is asymptotically normal with square-root convergence. As a result the joint convergence~\eqref{eq:jointresult} also holds. Note that though in this case the covariance matrix $V$ need not be block-diagonal and the limiting distribution $Z$ need not be block-independent. 

% \viet{Lemma~\ref{lemma:asymptotic} requires that $\wh \theta_c$ satisfies a central limit theorem.
% This condition $\sqrt{N_c}(\wh \theta_c-\theta_c)\xrightarrow{d.} \mathcal{N}(0,V)$ holds for $\Pnom$ in Example~\ref{ex:nominal-1} by i.i.d.~central limit theorem. Moreover, if $w_{MLE}$ that solves~\eqref{eq:MLE} is asymptotically normal with square-root convergence, then by the delta method, the condition $\sqrt{N_c}(\wh \theta_c-\theta_c)\xrightarrow{d.} \mathcal{N}(0,V)$ also holds for $\Pnom$ in Example~\ref{ex:nominal-2}.}

Finally, we study the structure of the worst-case distribution $\QQ\opt = \arg\max_{\QQ \in \mbb B(\Pnom)} \EE_{\QQ} \big[ \ell_\lambda(X, Y, w )\big] 
$
for any value of input $w$. This result explicitly quantifies how the adversary will generate the adversarial distribution adapted to any estimate $w$ provided by the statistician.

\begin{theorem}[Worst-case joint distribution] \label{thm:extreme}
	Given $\rho \in \R_+^C$ and $\eps \in \R_+$ such that $\sum_{c=1}^C \wh p_c \rho_c \le \eps$. For any $w$ and $c=1,\ldots,C$, let
	$\QQ_{Y|\wh x_c}\opt \sim f(\cdot | \theta_c\opt)$ with $\theta_c\opt = \wh \theta_c - \lambda(w, \wh x_c)/\dualvar_c\opt$, where $\dualvar_c\opt > 0$ is the solution of the nonlinear equation
    \be \notag%\label{eq:FOC}
        \Psi\big( \wh \theta_c - \dualvar^{-1}\lambda(w, \wh x_c) \big) + \dualvar^{-1} \inner{\nabla \Psi\big(\wh \theta_c - \dualvar^{-1} \lambda(w, \wh x_c) \big)}{\lambda(w, \wh x_c)} = \Psi(\wh \theta_c) - \rho_c,
    \ee
   and let $t_c\opt = \Psi(\lambda(w, \wh x_c)) - \inner{\nabla \Psi(\theta_c\opt)}{\lambda(w, \wh x_c)}$.
	Let $\alpha\opt \in \R$ and $\beta\opt \in \R_{++}$ be the solution of the following system of nonlinear equations
	\begin{subequations}
	\begin{align*}
		\textstyle\sum_{c=1}^C \wh p_c \exp \big( \beta^{-1}(t_c\opt - \alpha) - \rho_c - 1 \big) - 1 &= 0 \\
		\textstyle\sum_{c=1}^C \wh p_c (t_c\opt - \alpha) \exp \big( \beta^{-1}(t_c\opt - \alpha) - \rho_c - 1 \big) - (\eps + 1) \beta &= 0, 
	\end{align*}
	then the worst-case distribution is
	$\QQ\opt = \sum_{c=1}^C \wh p_c \exp\big( (\beta\opt)^{-1}(t_c\opt - \alpha\opt) - \rho_c -1 \big) \delta_{\wh x_c} \otimes \QQ_{Y|\wh x_c}\opt$.
	\end{subequations}
\end{theorem}
Notice that $\QQ\opt$ is decomposed into a worst-case marginal distribution of $X$ supported on $\wh x_c$ and a collection of worst-case conditional distributions $\QQ\opt_{Y|\wh x_c}$.

%%%%%%%%%%%%%%%%%%%%%%%%%%%%%%%%%%%%%%%%%%%%
\section{Numerical Experiments}
\begin{figure}[t]
    \centering
    \scalebox{0.44}{ \includegraphics{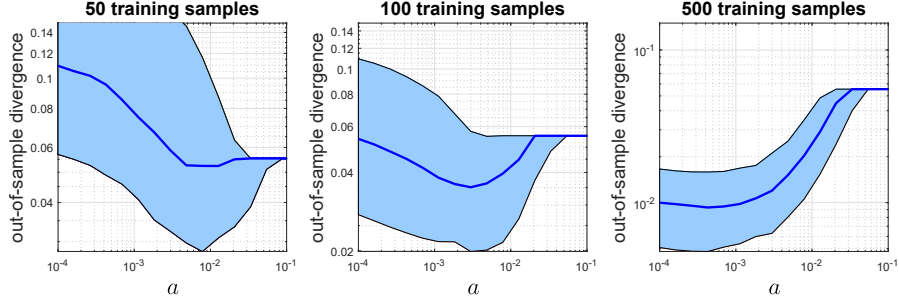}}
    \caption{Median (solid blue line) and the 10th-90th percentile region (shaded) of out-of-sample divergence loss collected from 100 independent runs.}
    \label{fig:1}
\end{figure}
\begin{table}[t]
	\centering
	\def\arraystretch{1}
	\begin{tabular}{|c|c|c|c|c|}
	\hline
	     && $N= 50$ & $N = 100$ & $N= 500$\\
		\hline
        &100(\text{DRO}-\text{MLE})/\text{MLE} & $-69.84\pm3.33\%$& $-52.66\pm3.98\%$ & $-22.25\pm4.11\%$\\
        $\text{CI}_{95\%}$ &$100(\text{DRO}-L_1)/L_1$ & $-22.59\pm4.02\%$ & $-19.90\pm4.16\%$& $-13.20\pm3.38\%$ \\
        &$100(\text{DRO}-L_2)/L_2$  & $-14.98\pm4.08\%$ & $-9.98\pm4.14\%$& $-5.62\pm 2.83\%$\\
        \hline
        \multirow{4}{3.5em}{$\text{CVaR}_{5\%}$} & MLE & 0.4906 & 0.1651&0.0246\\
        & {$L_1$} & 0.0967 &0.0742&0.0195\\
        & {$L_2$} & 0.0894 &0.0692&0.0176\\
        & {DRO} & 0.0547 &0.0518&0.0172\\
        \hline
	\end{tabular}
	\caption{Comparison between the DRO estimator with the other methods. Lower values are better.
	%Performance comparison of the DRO estimator with other methods. Lower value is better.
	}
	\label{table2} \vspace{-6mm}
\end{table}
We now showcase the abilities of the proposed framework in the distributionally robust Poisson and logistic regression settings using a combination of simulated and empirical experiments. All optimization problems are modeled in MATLAB using CVX~\cite{cvx} and solved by the exponential conic solver MOSEK~\cite{mosek} on an Intel i7 CPU (1.90GHz) computer. Optimization problems \eqref{eq:Poisson:refo:exp} and \eqref{eq:Logistic:refo:exp} are solved in under 3 seconds for all instances both in the simulated and empirical experiments. The MATLAB code is available at \url{https://github.com/angelosgeorghiou/DR-Parametric-MLE}.

\subsection{Poisson Regression}
We will use simulated experiments to demonstrate the behavior of the tuning parameters and to compare the performance of our estimator with regard to other established methods. We assume that the true distribution $\mathbb{P}$ is discrete, the $10$-dimensional covariate $X$ is supported on $K = 100$ points and their locations $\wh x_k$ are generated i.i.d.~using a standard normal distribution. We then generate a $K$-dimensional vector whose components are i.i.d.~uniform over $M_k\in[0, 10000]$, then normalize it to get the probability vector $p_k = M_k/M$ of the true marginal distribution of $X$. The value $w_0$ that determines the true conditional distribution $\PP_{Y|X}$ via the generative model~\eqref{eq:model} is assigned to $w_0 = \tilde w/ \| \tilde w\|_1$, where $\tilde w$ is drawn randomly from a 10-dimensional standard normal distribution.

Our experiment comprises 100 simulation runs. In each run we generate $N\in \{50,100,500\}$ training samples i.i.d.~from $\PP$ and use the MLE-compatible nominal distribution $\Pnom$ of the form~\eqref{eq:nominal} as in Example~\ref{ex:nominal-2}. We calibrate  the regression model  \eqref{eq:Poisson:refo:exp}  by tuning $\rho_c = a N_c^{-1}$  with $a\in[10^{-4},1]$ and $\eps\in [\sum_{c=1}^C \wh p_c \rho_c,1]$, both using a logarithmic scale with 20 discrete points.  The quality of an estimate $w\opt$ with respect to the true distribution $\PP$ is evaluated by the out-of-sample divergence loss
    \be \notag
        \notag \EE_{\PP_X}[\KL(\PP_{Y|X} \parallel \QQ_{w\opt, Y|X})] = \textstyle\sum_{k=1}^{K} p_k \big( \exp(w_0^\top \wh x_k) \big( (w_0-w\opt)^\top \wh x_k  - 1   \big) + \exp(\wh x_k^\top w\opt) \big).
    \ee  
    
In the first numerical experiment, we fix the marginal radius $\eps=1$ and examine how tuning the conditional radii $\rho_c$ can improve the quality of the estimator. Figure~\ref{fig:1} shows the 10th, 50th and 90th percentile of the out-of-sample divergence for different samples sizes. If the constant $a$ is chosen judiciously, incorporating the uncertainty in the conditional distribution can reduce the out-of-sample divergence loss by $17.65\%$, $10.55\%$ and $1.82\%$ for $N = 50, 100$ and $500$, respectively. 

Next, we compare the performance of our proposed estimator to the $w_{MLE}$ that solves~\eqref{eq:MLE} and the 1-norm ($L_1$) and 2-norm ($L_2$) MLE regularization, where the regularization weight takes values in $[10^{-4},1]$  on the logarithmic scale with 20 discrete points. In each run, we choose the optimal parameters that give the lowest of out-of-sample divergence for each method, and construct the empirical distribution of the out-of-sample divergence collected from 100 runs. Table~\ref{table2} reports the 95\% confidence intervals of $100(\text{DRO}-\text{MLE})/\text{MLE}$,   $100(\text{DRO}-L_1)/L_1$  and $100(\text{DRO}-L_2)/L_2$, as well as the 5\% Conditional Value-at-Risk (CVaR). Our approach delivers lower out-of-sample divergence loss compared to the other methods, and additionally ensures a lower value of CVaR for all sample sizes. This improvement is particularly evident in small sample sizes.

% \begin{table}[t]
% 	\centering
% 	\def\arraystretch{1}
% 	\begin{tabular}{|c|c|c|c|c||c|c|c|c|}
% 	\hline
% 	&\multicolumn{4}{c||}{AUC}&\multicolumn{4}{|c|}{CCR}\\
% 	\hline
% 	Dataset &	DRO	&	$L_1$	&	$L_2$ &	MLE	&	DRO	&	$L_1$	&	$L_2$ & MLE	\\
% \hline
% australian	&	\textbf{92.74}	&	92.73	&	92.71	&	92.61	&	\textbf{85.75}	&	85.52	&	85.6	&	85.72	\\
% banknote authentication	&	\textbf{98.46}	&	98.43	&	98.45	&	98.45	&	94.31	&	94.16	&	\textbf{94.35}	&	94.32	\\
% climate model	&	94.31   	&	\textbf{94.85}	&	94.13	&	82.76	&	\textbf{95.04}	&	94.85	&	94.82	&	93.89	\\
% german credit&	\textbf{75.75}&	75.74&	75.74&	75.67&	73.86	&73.82&	73.70&	\textbf{74.05}\\
% haberman	&	66.86	&	\textbf{69.19}	&	68.17	&	67.2	&	\textbf{73.83}	&	73.19	&	73.18	&	73.79	\\
% housing	&	\textbf{76.24}	&	75.37	&	75.57	&	75.73	&	91.65	&	\textbf{92.68}	&	92.64	&	91.70	\\
% ILPD	&	\textbf{74.01}	&	73.56	&	73.77	&	73.66	&	71.11	&	71.68	&	\textbf{71.79}	&	71.07	\\
% mammographic mass	& \textbf{	87.73}	&	87.70	&	87.68	&	87.71	&	80.99	&	80.99	&	80.94	&	\textbf{81.20}	\\
% \hline
% 	\end{tabular}
% 	\caption{Average area under the curve (AUC) and correct classification rates (CCR) on UCI datasets. }
% 	\label{table:datasets}\vspace{-6mm}
% \end{table}

\begin{table}[h]
	\centering
	\resizebox{1\columnwidth}{!}{
	\begin{tabular}{|l|c|c|c|c|c||c|c|c|c|c|}
	\hline
	&\multicolumn{5}{c||}{AUC}&\multicolumn{5}{|c|}{CCR}\\
	\hline
	Dataset &	DRO & KL	&	$L_1$	&	$L_2$ &	MLE	&	DRO &  KL	&	$L_1$	&	$L_2$ & MLE	\\
\hline
australian ($N=690,\,n = 14$) 	&	\textbf{92.74}	&	92.62	&	92.73	&	92.71	&	92.61	&	\textbf{85.75}	&	85.72	&	85.52	&	85.60	&	85.72	\\
banknote ($N=1372,\,n = 4$)	&	\textbf{98.46}	&	\textbf{98.46}	&	98.43	&	98.45	&	98.45	&	94.31	&	94.32	&	94.16	& \textbf{94.35}	&	94.32	\\
climate ($N=540,\,n=18$)	&	94.30	&	82.77	&	\textbf{94.85}	&	94.13	&	82.76	&	\textbf{95.04}	&	93.89	&	94.85	&	94.83	&	93.89	\\
german ($N=1000,n=19$)	&	\textbf{75.75}	&	75.68	&	75.74	&	75.74	&	75.67	&	73.86	&	\textbf{74.05}	&	73.82	&	73.70	&	\textbf{74.05}	\\
haberman ($N=306,n=3$)	&	66.86	&	67.21	&	\textbf{69.19}	&	68.17	&	67.20	&	\textbf{73.83}	&	73.80	&	73.20	&	73.18	&	73.80	\\
housing ($N=506,\,n=13$)	&	\textbf{76.24}	&	75.73	&	75.37	&	75.57	&	75.73	&	91.65	&	91.70	&	\textbf{92.68}	&	92.65	&	91.70	\\
ILPD ($N=583,\,n=10$)	&	\textbf{74.01}	&	73.66	&	73.56	&	73.77	&	73.66	&	71.11	&	71.07	&	71.68	&	\textbf{71.79}	&	71.07	\\
mammo. ($N=830\,n = 5$)	&	\textbf{87.73}	&	87.72	&	87.70	&	87.68	&	87.71	&	81.00	&	\textbf{81.20}	&	80.99	&	80.94	&	\textbf{81.20}	\\
\hline
	\end{tabular}
	}
	\caption{Average area under the curve (AUC) and correct classification rates (CCR) on UCI datasets ($m = 1$). }
	\label{table:datasets}
\end{table}

%%%%%%%%%%%%%%%%%%%%%%%%
\vspace{-2mm}
\subsection{Logistic Regression}

We now study the performance of our proposed estimation in a classification setting using data sets from the  UCI repository~\cite{Dua:2019}. We compare four different models: our proposed DRO estimator~\eqref{eq:Logistic:refo:exp}, the $w_{MLE}$ that solves~\eqref{eq:MLE}, the 1-norm ($L_1$) and 2-norm ($L_2$) MLE regularization. In each independent trial, we randomly split the data into train-validation-test set with proportion 50\%-25\%-25\%. For our estimator, we calibrate  the regression model  \eqref{eq:Logistic:refo:exp} by tuning $\rho_c = a N_c^{-1}$  with $a\in[10^{-4},10]$ using a logarithmic scale with 10 discrete points and setting $\eps=2\sum_{c=1}^C \wh p_c \rho_c$.
Similarly, for the $L_1$ and $L_2$ regularization, we calibrate the regularization weight from $[10^{-4},1]$  on the logarithmic scale with 10 discrete points. Leveraging Proposition~\ref{prop:reweighting}, we also compare our approach versus the DRO \textit{non}parametric Kullback-Leibler (KL) MLE by setting $\rho_c = 0$ and tune only with $\varepsilon \in [10^{-4}, 10]$ with 10 logarithmic scale points. The performance of the methods was evaluated on the testing data using two popular metrics: the correct classification rate (CCR) with a threshold level of 0.5, and the area under the receiver operating characteristics curve (AUC). Table~\ref{table:datasets} reports the performance of each method averaged over 100 runs. One can observe that our estimator performs reasonably well compared to other regularization techniques in both performance metrics.

\begin{remark}[Uncertainty in $\wh x_c$]
    The absolute continuity condition of the KL divergence implies that our proposed model cannot hedge against the error in the covariate $\wh x_c$. It is natural to ask which model can effectively cover this covariate error. Unfortunately, answering this question needs to overcome to technical difficulties: first, the log-partition function $\Psi$ is convex; second, the there are multiplicative terms between $X$ and $Y$ in the objective function. Maximizing over the $X$ space to find the worst-case covariate is thus difficult. Alternatively, one can think of perturbing each $\widehat x_c$ in a finite set but this approach will lead to trivial modifications of the constraints of problem~\eqref{eq:refor}.
\end{remark}

% To calculate the loss, sample $\wh x_k$ from some distribution of $\PP_X$. The loss can be calculated as
% \[
%     \sum_{k} p_k \left(\frac{(w_{\text{true}} - w\opt)^\top \wh x_k}{1 + \exp(-\wh x_k^\top w_{\text{true}})} + \log \big( \frac{1+\exp((w\opt)^\top \wh x_k)}{1 + \exp(w_{\text{true}}^\top \wh x_k)} \big)  \right)
% \]
\begin{comment}
\begin{table}[ht]
	\centering
	\def\arraystretch{1}
	\begin{tabular}{|c|c|c|c|c|}
	    
		\hline
        &MLE & $L_1$ & $L_2$ & DRO\\
		\hline
         CI min logloss &   $0.0879\pm0.0122$ & $0.0552\pm0.0045$& $0.0434\pm0.0036$ & $0.0409 \pm   0.0033$\\
         CVaR 95\% min logloss & $0.2966$ & $0.1163$ & $0.0896$ & $0.0826$\\
		\hline
	\end{tabular}
\end{table}

\begin{table}[ht]
	\centering
	\def\arraystretch{1}
	\begin{tabular}{|c|c|c|c|}
	\hline
	     & CI min logloss & CI CCR & CI AUC\\
		\hline
        (DRO - MLE)/MLE & $-0.4433\pm 0.0469$ & $0.0075\pm0.0013$ & $0.0067\pm0.0015$\\
        (DRO - $L_1$)/$L_1$ & $-0.2449\pm0.0317$ & $0.0067\pm0.0013$& $0.0058\pm0.0012$ \\
        (DRO - $L_2$)/$L_2$ &$-0.0392\pm0.0340$ & $0.0041\pm0.0015$& $0.0012\pm0.0017$\\
        \hline
	\end{tabular}
\end{table}

\end{comment}

\paragraph{\bf Acknowledgments.}
Material in this paper is based upon work supported by the Air Force Office of Scientific Research under award number FA9550-20-1-0397. Additional support is gratefully acknowledged from NSF grants 1915967, 1820942, 1838676 and from the China Merchant Bank.

\newpage
\appendix

This appendix is organized as follows. Section~\ref{sec:app1}-\ref{sec:app3} provide the detailed proofs for all the technical results in the main paper. Section~\ref{sec:app-aux} provides further discussion on the variance regularization surrogate result in Proposition~\ref{prop:surrogate}.

\section{Proofs of Section~\ref{sect:ambiguity-set}}
\label{sec:app1}

\begin{proof}[Proof of Example~\ref{ex:nominal-1}]
We note that 
\begin{align*}
    \Min{w \in \mc W}~\EE_{\Pnom}[ \ell_{\lambda}(X, Y, w)]& =\Min{w \in \mc W}\sum_{c=1}^C\wh p_c\left(\Psi(\lambda(w,\wh x_c))-\inner{\EE_{\Pnom_{Y|\wh x_c}}[T(Y)]}{\lambda(w,\wh x_c)}\right)\\
    & = \Min{w \in \mc W}\sum_{c=1}^C\wh p_c\left(\Psi(\lambda(w,\wh x_c))-\inner{\nabla \Psi(\wh\theta_c)}{\lambda(w,\wh x_c)}\right).
\end{align*}
If  $\wh \theta_c = (\nabla \Psi)^{-1}\big((N_c)^{-1} \sum_{\wh x_i = \wh x_c} T(\wh y_i) \big)$, then we have
\begin{align*}
\Min{w \in \mc W}~\EE_{\Pnom}[ \ell_{\lambda}(X, Y, w)]&=\Min{w \in \mc W}\sum_{c=1}^C\wh p_c\left(\Psi(\lambda(w,\wh x_c))-\inner{\frac{\sum_{\wh x_i = \wh x_c} T(\wh y_i)}{N_c}}{\lambda(w,\wh x_c})\right)\\
&=\Min{w \in \mc W} \frac{1}{N} \sum_{i=1}^{N} \left( \Psi( \lambda(w, \wh x_i)) - \inner{T(\wh y_i)}{ \lambda(w, \wh x_i)} \right),
\end{align*}
where we used $\wh p_c = N_c/N$. Therefore $w_{MLE}$ solves $\min_{w \in \mc W}~\EE_{\Pnom}[ \ell_{\lambda}(X, Y, w)]$.
\end{proof}

\begin{proof}[Proof of Example~\ref{ex:nominal-2}]
We find
\begin{align*}
\Min{w \in \mc W}~\EE_{\Pnom}[ \ell_{\lambda}(X, Y, w)]&= \Min{w \in \mc W} \sum_{c=1}^C\wh p_c\left(\Psi\big(\lambda(w,\wh x_c) \big)-\inner{\EE_{\Pnom_{Y|\wh x_c}}[T(Y)]}{\lambda(w,\wh x_c)}\right) \\
&\ge \sum_{c=1}^C\wh  p_c \Min{w_c \in \mc W} \left(\Psi\big(\lambda(w_c,\wh x_c) \big)-\inner{\EE_{\Pnom_{Y|\wh x_c}}[T(Y)]}{\lambda(w_c,\wh x_c)}\right) \\
&= \sum_{c=1}^C\wh  p_c  \left(\Psi\big(\lambda(w_{MLE},\wh x_c) \big)-\inner{\EE_{\Pnom_{Y|\wh x_c}}[T(Y)]}{\lambda(w_{MLE},\wh x_c)}\right),
\end{align*}
where the first equality follows from the definition of the log-loss function $\ell_\lambda$, the inequality follows because $\wh p_c > 0$, and the last equality follows because of the convex conjugate relationship that implies the optimal solution $w_c\opt$ should satisfy
\[
    \nabla \Psi\big(\lambda(w_c\opt, \wh x_c)\big) = \EE_{\Pnom_{Y|\wh x_c}}[T(Y)] = \nabla \Psi\big(\lambda(w_{MLE}, \wh x_c) \big) \implies w_c\opt = w_{MLE}.
\]
This implies that $w_{MLE}$ solves $\min_{w \in \mc W}~\EE_{\Pnom}[ \ell_{\lambda}(X, Y, w)]$ and completes the proof.
\end{proof}
%\viet{The above qed square is taking the whole line. Please delete the commented part}

\begin{proof}[Proof of Proposition~\ref{prop:conservative}]
	Fix any set of conditional radii $\rho \in \R_+^C$. If $\mbb B_{\eps, \rho}(\Pnom)$ is empty then it is trivial that $\mbb B_{\eps, \rho}(\Pnom) \subset \mc B_{\eps}(\Pnom)$. Suppose that $\mbb B_{\eps, \rho}(\Pnom)$ is non-empty and pick any $\QQ \in \mbb B_{\eps, \rho}(\Pnom)$. By definition of the set $\mbb B_{\eps, \rho}(\Pnom)$, $\QQ$ can be decomposed into a marginal $\QQ_X$ and a collection of conditional measures $\QQ_{Y|\wh x_c}$. Furthermore, because $\eps$ is finite, the marginal $\QQ_X$ should be absolutely continuous with respect to $\Pnom_X$. We have
	\begin{align*}
		\KL(\QQ \parallel \Pnom) &=  \KL(\QQ_X \parallel \Pnom_X) + \EE_{\QQ_X}[\KL(\QQ_{Y|X} \parallel \Pnom_{Y|X})] \\
		&\le \KL(\QQ_X \parallel \Pnom_X) + \EE_{\QQ_X}[\sum_{c=1}^C \rho_c \mathbbm{1}_{\wh x_c}(X)] \le \eps,
	\end{align*}
	where the equality is from the chain rule of the conditional relative entropy~\cite[Lemma~7.9]{ref:gray2011entropy}. The first inequality follows from the fact that $\KL(\QQ_{Y|\wh x_c} \parallel \Pnom_{Y|\wh x_c}) \le \rho_c$ for every $c$. The second inequality follows from the last constraint defining the set $\mbb B_{\eps, \rho}(\Pnom)$. This implies that $\QQ \in \mc B_{\eps}(\Pnom)$, and because $\QQ$ was chosen arbitrarily, we have $\mbb B_{\eps, \rho} \subseteq \mc B_{\eps}(\Pnom)$. As a consequence, $\bigcup_{\rho \in \R_+^C} \mbb B_{\eps, \rho}(\Pnom) \subseteq \mc B_{\eps}(\Pnom)$.
	
	Regarding the reverse relation, pick an arbitrary $\QQ \in \mc B_\eps(\Pnom)$ which admits the decomposition into a marginal $\QQ_X$ and conditional measures $\QQ_{Y|\wh x_c}$. By setting the conditional radii $\rho \in \R_+^C$ with $\rho_c = \KL(\QQ_{Y|\wh x_c} \parallel \Pnom_{Y|\wh x_c})$ for every $c$, one can verify using the chain rule of the conditional relative entropy that $\QQ \in \mbb B_{\eps, \rho}(\Pnom)$. This implies that $\mbb B_{\eps, \rho}(\Pnom) \subseteq \bigcup_{\rho \in \R_+^C} \mbb B_{\eps, \rho}(\Pnom)$.
	
	Concerning the last statement, notice that the condition~$\sum_{c=1}^C \wh p_c \rho_c \leq \eps$ implies that $\Pnom \in \mbb B_{\eps, \rho}(\Pnom)$ and thus $\mbb B_{\eps, \rho}(\Pnom)$ is non-empty. The proof is complete.
\end{proof}

%%%%%%%%%%%%%%%%%%%%%%%%%%%%
\section{Proofs of Section~\ref{sect:refor}}

The proof of Proposition~\ref{prop:refor} relies on the following preliminary result.

\begin{lemma} \label{lemma:support}    Let $\wh p \in \R_{++}^C$ be a probability vector summing up to one. For any $\eps \in \R_+$ and $\rho \in \R_+^C$ satisfying $\sum_{c=1}^C \wh p_c \rho_c \leq \eps$, the finite dimensional set
\be \label{eq:Q-def}
    \mc Q \Let \left\{q \in \R_+^C: \sum_{c =1}^C q_c = 1, ~\ds \sum_{c =1}^C q_c (\log q_c - \log \wh p_c + \rho_c) \leq \eps \right\}
\ee
is compact and convex. Moreover, the support function $h_{\mc Q}$ of $\mc Q$ satisfies
\[
    \forall t \in \R^C: \quad h_{\mc Q}(t) \Let \Sup{q \in \mc Q}~q^\top t = \Inf{\alpha \in \R,~\beta \in \R_{++}}~\left\{ \alpha + \beta\eps + \beta \sum_{c=1}^C \wh p_c \exp \Big( \frac{t_c - \alpha}{\beta} - \rho_c - 1 \Big) \right\}.
\]
\end{lemma}
\begin{proof}[Proof of Lemma~\ref{lemma:support}]
    The function $\R_+^C \ni q \mapsto \sum_{c=1}^C q_c (\log q_c - \log \wh p_c + \rho_c) \in \R_+$ is continuous and convex, hence, the set $\{q \in \R_+^C: \sum_{c=1}^C q_c (\log q_c - \log \wh p_c + \rho_c) \le \eps\}$ is closed and convex. Consequentially, $\mc Q$ can be written as the intersection between a simplex (thus compact and convex) and a closed, convex set, so $\mc Q$ is compact and convex. 
    
    The proof of the support function of $\mc Q$ proceeds in 2 steps. First, we prove the support function for the $\epsilon$-inflated set
    \[
        \mc Q_\epsilon = \left\{q \in \R_+^C: \sum_{c =1}^C q_c = 1,~\ds \sum_{c =1}^C q_c (\log q_c - \log \wh p_c + \rho_c) \leq \eps + \epsilon \right\}
    \]
    with the right-hand side of the last constraint being inflated with $\epsilon \in \R_{++}$. In the second step, we use a limit argument to show that the support function of $\mc Q$ is attained as the limit of the support function of $\mc Q_\epsilon$ as $\epsilon$ tends to 0.
    
    Reminding that $\Delta$ is the $C$-dimensional simplex. For any $t \in \R^C$ and any $\epsilon \in \R_{++}$, by the definition of the support function, we have for every $t \in \R^C$
    \begin{subequations}
    \begin{align}
        h_{\mc Q_\epsilon}(t) &= \left\{
        \begin{array}{cl}
            \sup & q^\top t \\
            \st  & q \in \Delta,~\ds \sum_{c =1}^C q_c (\log q_c - \log \wh p_c + \rho_c) \leq \eps + \epsilon
        \end{array}
        \right. \label{eq:support-1} \\
        &= \Sup{q \in \Delta}~\Inf{\beta\in \R_+}~q^\top t + \beta(\eps + \epsilon - \sum_{c =1}^C q_c \big(\log q_c - \log \wh p_c + \rho_c \big) \notag \\
        &= \Inf{\beta\in \R_+}~\Sup{q \in \Delta}~q^\top t + \beta(\eps + \epsilon - \sum_{c =1}^C q_c \big(\log q_c - \log \wh p_c + \rho_c \big), \label{eq:support-2}
    \end{align}
    \end{subequations}
    where the interchange of the sup-inf operators in~\eqref{eq:support-2} is justified by strong duality~\cite[Proposition~5.3.1]{ref:bertsekas2009convex} because $\wh p$ constitutes a Slater point of the set $\mc Q_\epsilon$. By Berge's maximum theorem~\cite{ref:berge1963topological}, the optimal value of the inner supremum problem is a continuous function in $\beta$ because the simplex $\Delta$ is compact and the objective function is continuous in the decision variable $q$. As a consequence, we can restrict $\beta \in \R_{++}$ without any loss of optimality. Because $\Delta$ is prescribed using linear constraints, strong duality implies that 
    \begin{align*}
         h_{\mc Q_\epsilon}(t)&= \Inf{\alpha \in \R,~\beta \in \R_{++}}\left\{ \alpha + \beta(\eps + \epsilon) + \Sup{q \in \R_+^C} ~\sum_{c=1}^C q_c (t_c - \alpha + \beta \log \wh p_c - \beta \rho_c - \beta\log q_c) \right\} \notag \\
        &= \Inf{\alpha \in \R,~\beta \in \R_{++}}\left\{ \alpha + \beta(\eps + \epsilon) + ~\sum_{c=1}^C \Sup{q_c \in \R_+}  q_c (t_c - \alpha + \beta \log \wh p_c - \beta \rho_c - \beta \log q_c) \right\}, \label{eq:support-2}
    \end{align*}
    where the last equality holds because the supremum problem is separable in each decision variable $q_c$. It now follows from the first-order optimality condition that the maximizer $q_c\opt$ is
    \[
        q_c\opt = \exp\Big( \frac{t_c - \alpha + \beta \log \wh p_c - \beta \rho_c -\beta}{\beta} \Big) > 0,
    \]
    and by substituting this maximizer into the objective function, the value of the support function $h_{\mc Q_\epsilon}(t)$ is then equal to the optimal value of the below optimization problem
    \[
        \Inf{\alpha \in \R,~\beta \in \R_{++}}~\alpha + \beta(\eps + \epsilon) + \beta \sum_{c=1}^C \wh p_c \exp \Big( \frac{t_c - \alpha}{\beta} - \rho_c - 1 \Big).
    \] 
    
    We now proceed to the second step. Denote temporarily the objective function of the above problem as $G(\epsilon,\gamma)$, where $\gamma = [\alpha; \beta]$ combines both dual variables $\alpha$ and $\beta$. Define the function
    \[
        g(\epsilon) = \Inf{\gamma \in \Gamma}~ G(\epsilon, \gamma), \qquad \text{with } \Gamma \Let \R\times \R_{++}.
    \]
    Because $G$ is continuous, \cite[Lemma~2.7]{ref:nguyen2018distributionally} implies that $g$ is upper-semicontinuous at 0. Furthermore, $G$ is calm from below at $\epsilon = 0$ because $G(\epsilon, \gamma) - G(0, \gamma) = \beta \epsilon \ge 0$,
    thus \cite[Lemma~2.7]{ref:nguyen2018distributionally} implies that $g$ is lower-semicontinuous at 0. These two facts lead to the continuity of $g$ at $0$. From the first part of the proof, we have $g(\epsilon) = h_{\mc Q_\epsilon}(t)$ for any $\epsilon \in \R_+$. Moreover, by applying Berge's maximum theorem~\cite{ref:berge1963topological} to~\eqref{eq:support-1}, $h_{\mc Q_\epsilon}(t)$ is a continuous function of $\epsilon$ over $\R_+$. Thus we find
    \[
        h_{\mc Q}(t) = h_{\mc Q_0}(t) = \lim_{\epsilon \downarrow 0} h_{\mc Q_\epsilon}(t) = \lim_{\epsilon \downarrow 0} g(\epsilon) = g(0),
    \]
    where the chain of equalities follows from the definition of $\mc Q_\epsilon$, the continuity of $h_{\mc Q_\epsilon}(t)$ in $\epsilon$, the fact that $g(\epsilon) = h_{\mc Q_\epsilon}(t)$ for $\epsilon > 0$, and the continuity of $g$ at 0 established previously. The proof is now completed.
\end{proof}

\begin{proof}[Proof of Proposition~\ref{prop:refor}]
To facilitate the proof, we define the following ambiguity set over the marginal distribution of the covariate $X$ as
\[
	\mbb B_X \Let \left\{ \QQ_{X} \in \mc M (\mc X): \KL(\QQ_X \parallel \Pnom_{X} ) + \EE_{\QQ_X}[\sum_{c=1}^C \rho_c \mathbbm{1}_{\wh x_c}(X)] \leq \eps \right\}.
\]
Given a nominal marginal distribution $\Pnom_X$ supported on a finite set $\{\wh x_c\}_{c \in \C}$, the absolute continuity requirement suggests that $\KL(\QQ_X \parallel \Pnom_X)$ is finite if and only if $\QQ_X$ is absolutely continuous with respect to $\Pnom_X$. Thus, any $\QQ_X$ of interest should be supported on the same set  $\{\wh x_c\}_{c = 1, \ldots, C}$, and $\QQ_X$ and be finitely parametrized by a $C$-dimensional vector $\{q_c\}_{c = 1, \ldots, C}$. Let $\mc Q$ denote the convex compact feasible set in $\R^C$, that is, 
\[
\mc Q \Let \left\{q \in \R_+^C: \sum_{c =1}^C q_c = 1, \ds \sum_{c =1}^C q_c (\log q_c - \log \wh p_c + \rho_c) \leq \eps \right\},
\]
and the ambiguity set $\mbb B_X$ can now be finitely parametrized as
\[
    \mbb B_X = \left\{ \QQ_{X} \in \mc M (\mc X): \exists q \in \mc Q,~\QQ_{X} = \sum_{i=1}^C q_c \delta_{\wh x_c}
    \right\}.
\]
By coupling $\mbb B_X$ with the conditional ambiguity sets $\mbb B_{Y|\wh x_c}$, $\mbb B(\Pnom)$ can be re-written as
\[
	\mbb B(\Pnom) = \left\{ \QQ\in \mc M(\mc X \times \mc Y): 
	\begin{array}{l}
		\exists \QQ_X \in \mbb B_X,~\QQ_{Y|\wh x_c} \in \mbb B_{Y| \wh x_c} \quad \forall c = 1, \ldots, C \\
		\QQ (\{\wh x_c\} \times A) = \QQ_X(\{\wh x_c\}) \QQ_{Y|\wh x_c}(A) ~ \forall A \in \mc F(\mc Y) \quad \forall c = 1, \ldots, C
	\end{array}
	\right\}
\]
The worst-case expected loss becomes
\begin{align*}
    \Sup{\QQ \in \mbb B(\Pnom)}~\EE_{\QQ}[L(X, Y)] &=	\Sup{\QQ_{X} \in \mbb B_{X}} \EE_{\QQ_{X}}\left[  \Sup{\QQ_{Y|X} \in \mbb B_{Y|X}} \EE_{\QQ_{Y|X}} \left[ L(X, Y) \right] \right] \\
    &= \Sup{q \in \mc Q}~\sum_{c = 1}^C q_c \Sup{\QQ_{Y|\wh x_c} \in \mbb B_{Y|\wh x_c}} \EE_{\QQ_{Y|\wh x_c}} \left[ L(\wh x_c, Y) \right],
\end{align*}
where the first equality follows from the law of total expectation, and the second equality follows from the finite reparametrization of $\mbb B_X$. If we denote by $\mc T$ the epigraph reformulation of the worst-case conditional expectations
\[
\mc T \Let \left\{ t \in \R^C:
\begin{array}{ll}
\Sup{\QQ_{Y|\wh x_c} \in \mbb B_{Y|\wh x_c}} \EE_{\QQ_{Y|\wh x_c}} \left[ L(\wh x_c, Y) \right] \leq t_c &\forall c =1, \ldots, C
\end{array}
\right\},
\]
then the worst-case expected loss can be further re-expressed as
\begin{subequations}
\begin{align}
\Sup{\QQ \in \mbb B(\Pnom)}~\EE_{\QQ}[ L(X, Y)]&= \Sup{q\in \mc Q}  \Inf{t \in \mc T} ~q^\top t \label{eq:proof:1}\\
&= \Inf{t \in \mc T} \Sup{q\in \mc Q}  ~q^\top t \label{eq:proof:2}\\
&= \left\{ \begin{array}{cl}
	\inf & \alpha + \beta \eps + \beta \ds \sum_{c=1}^C \wh p_c \exp\left( \frac{t_c - \alpha}{\beta} - \rho_c - 1 \right)\\
	\st & t \in \mc T, \; \alpha \in \R,\; \beta \in \R_{++},
	\end{array} \right. \label{eq:proof:3}
\end{align}
\end{subequations}
where the sup-inf formulation~\eqref{eq:proof:1} is justified because $q$ is non-negative and we can resort to the epigraph formulations of the worst-case conditional expected loss. In~\eqref{eq:proof:2} we applied Sion's minimax theorem~\cite{ref:sion1958minimax}, which is valid because the sup-inf program~\eqref{eq:proof:1} is a concave-convex saddle problem, and $\mc Q$ is convex and compact and $\mc T$ is convex. In~\eqref{eq:proof:3} we have used Lemma~\ref{lemma:support} to reformulate the supremum over $q$. The claim then follows.
\end{proof}

Instead of solving the problem in the natural parameters $\theta$ coupled with its log-partition function $\Psi$, we will use the reparametrization to the mean parameters using the conjugate function of $\Psi$. More specifically, let $\phi$ be the convex conjugate of $\Psi$, that is,
\[
	\phi : \mu \mapsto \Sup{\theta \in \Theta} \left\{ \inner{\mu}{\theta} - \Psi(\theta) \right\}
\]

% maybe we need a lemma that collects all the necessary results linking $\phi$, $\Psi$, $\mu$ and $\theta$. Some important relationship is
% \begin{itemize}
%     \item $\phi$ is convex over $\mathrm{dom}(\phi)$
%     \item $\mu(\theta) = \nabla \Psi(\theta)$ and $\theta(\mu) = \nabla \phi(\mu)$
%     \item $\inner{\wh \mu_c}{\wh \theta_c} - \phi(\wh \mu_c) = \Psi(\wh \theta_c)$
%     \item domain of $\phi$ should be $\dom(\phi)$, and $\mu \in \inte(\dom(\phi))$. So far no result saying that $\dom(\phi)$ is open
%     \item check if $\phi$ is strictly convex?
%     \item check if $\phi$ is twice differentiable
%     \item check if $\theta \mapsto \nabla \Psi(\theta)$ is continuous over $\Theta$.
% \end{itemize}

Before proceeding to the technical proofs, the below lemma collects from the existing literature the necessary background knowledge about the log-partition function $\Psi$ and its conjugate $\phi$, along with the relationship between the natural parameter $\theta$ and its corresponding expectation parameter $\mu$.

\begin{lemma}[Relevant facts] \label{lemma:facts} The following assertions hold for regular exponential family.
\begin{enumerate}[label=(\roman*), leftmargin=6mm]
    \item \label{fact:0}
    The function $\phi$ is closed, convex and proper on $\mathbb R^p$.
    \item \label{fact:1}
    $(\Theta,\Psi)$ and $\left(\inte(\dom(\phi)),\phi\right)$ are convex functions of Legendre type, and they are Legendre duals of each other. 
    \item \label{fact:2} The gradient function $\nabla \Psi$ is a  one-to-one function from the open convex set $\Theta$ onto the open convex set $\inte(\dom(\phi))$.
    \item \label{fact:3}
    The gradient functions $\nabla \Psi$ and $\nabla \phi$ are continuous, and $\nabla \phi=(\nabla\Psi)^{-1}$.
    \item \label{fact:4} The function $\phi$ is essentially smooth over $\inte(\dom(\phi))$.
    \end{enumerate}
\end{lemma}
\begin{proof}[Proof of Lemma~\ref{lemma:facts}] 
Assertion~\ref{fact:0} holds since $\inner{\mu}{\theta}-\Psi(\theta)$ is convex and closed for each $\theta$, thus taking supremum, $\phi$ is convex and closed. $\phi$ is proper since $\dom(\phi)$ is non-empty. Assertions~\ref{fact:1} to~\ref{fact:3} follows from~\cite[Lemma~1]{ref:banerjee2005clustering} and~\cite[Theorem~2]{ref:banerjee2005clustering}. Assertion~\ref{fact:4} follows from~\cite[Lemma~1]{ref:banerjee2005clustering} and~\cite[Theorem~26.3]{ref:rockafellar1970convex}, and the fact that $\Psi$ and $\phi$ is a convex conjugate pair.
\end{proof}

From Assertion~\ref{fact:1}, we have the mappings between the dual spaces $\inte(\dom(\phi))$ and $\Theta$ are given by the Legendre transformation
\[
\mu(\theta)=\nabla\Psi(\theta)\quad\text{and}\quad\theta(\mu)=\nabla\phi(\mu).
\]
For any $\mu\in\inte(\dom(\phi))$, the conjugate function $\phi$ can be expressed as \[
\phi(\mu)=\inner{\mu}{\theta(\mu)}-\Psi(\theta(\mu)).
\]
\begin{lemma}[KL divergence between distributions from exponential family] \label{lemma:KL-exp} Suppose that $\QQ_1$ and $\QQ_2$ belong to the exponential family of distributions with the same log-partition function $\Psi$ and with natural parameters $\theta_1$ and $\theta_2$ respectively. The KL divergence from $\QQ_1$ to $\QQ_2$ amounts to
	\begin{align*}
		\KL (\QQ_1 \parallel \QQ_2) & = \inner{\theta_1 - \theta_2}{\mu_1} - \Psi(\theta_1) + \Psi(\theta_2) 
		= \phi(\mu_1) - \phi(\mu_2) - \inner{\mu_1 - \mu_2}{\theta_2},
	\end{align*}
	where $\phi$ is the convex conjugate of $\Psi$, and $\mu_j = \nabla \Psi(\theta_j)$ for any $j \in \{1, 2\}$.
\end{lemma}
The result of Lemma~\ref{lemma:KL-exp} can be found in~\cite[Appendix~A]{ref:banerjee2005clustering}, but the explicit proof is included here for completeness.
\begin{proof}[Proof of Lemma~\ref{lemma:KL-exp}]
	One finds
	\begin{subequations}
	\begin{align}
		\KL (\QQ_1 \parallel \QQ_2) &= \EE_{\QQ_1}[ \log (\mathrm{d} \QQ_1 / \mathrm{d} \QQ_2)] \notag \\
		&= \EE_{\QQ_1}[ \inner{T(Y)}{\theta_1 - \theta_2} - \Psi(\theta_1) + \Psi(\theta_2)] \label{eq:KL-exp1} \\
		&=  \inner{ \mu_1 }{\theta_1 - \theta_2} - \Psi(\theta_1) + \Psi(\theta_2), \label{eq:KL-exp2}
	\end{align}
	\end{subequations}
	where equality~\eqref{eq:KL-exp1} follows by calculating the logarithm of the Radon-Nikodym derivatives between two distributions, and equality~\eqref{eq:KL-exp2} follows by noting that $\mu_1 = \EE_{\QQ_1}[T(Y)]$.
	
	By \cite[Theorem~4]{ref:banerjee2005clustering}, one can also rewrite the density using the mean parameter $\mu = \mu(\theta)$ as
	\begin{align*}
	f(y|\mu)&=h(y)\exp\left(\inner{\theta}{T(y)}-\Psi(\theta)\right)\\
	& = h(y) \exp\left(\phi(\mu)+\inner{T(y)-\mu}{\nabla\phi(\mu)}\right)
	\end{align*}
	The KL divergence from $\QQ_1$ to $\QQ_2$ amounts to
	\begin{subequations}
		\begin{align}
		\KL (\QQ_1 \parallel \QQ_2) &= \EE_{\QQ_1}[ \log (\mathrm{d} \QQ_1 / \mathrm{d} \QQ_2)] \notag \\
		&= \EE_{\QQ_1}[ \phi(\mu_1) - \phi(\mu_2) + \inner{T(Y)}{\nabla \phi(\mu_1) - \nabla \phi(\mu_2) } - \inner{\mu_1}{\nabla \phi(\mu_1)} + \inner{\mu_2}{\nabla \phi(\mu_2)}] \label{eq:KL-exp3} \\
		&=  \inner{\mu_2 - \mu_1}{\theta_2} + \phi(\mu_1) - \phi(\mu_2).\label{eq:KL-exp4}
		\end{align}
	\end{subequations}
	From Assertion~\ref{fact:3} in Lemma~\ref{lemma:facts}, we notice that $\theta_2=\nabla \phi(\mu_2)$, which completes the proof.
\end{proof}

Recall that the conditional ambiguity set defined in~\eqref{eq:cond-ambi} is
\[
	\mbb B_{Y|\wh x_c} \Let \left\{ \QQ_{Y|\wh x_c} \in \mc M(\mc Y): \exists \theta \in \Theta,~\QQ_{Y|\wh x_c}(\cdot) \sim f(\cdot | \theta),~\KL(\QQ_{Y|\wh x_c}  \parallel  \Pnom_{Y|\wh x_c}) \leq \rho_c  \right\}
\]
for a parametric, nominal conditional measure $\Pnom_{Y|\wh x_c} \sim f(\cdot | \wh \theta_c)$, $\wh \theta_c \in \Theta$ and a radius $\rho_c \in \R_+$.
The uncertainty set $\mc S_c$ of expectation parameters induced by the ambiguity set $\mbb B_{Y|\wh x_c}$ is defined as
\[
\mc S_c \Let \left\{ \mu \in \dom(\phi): \exists \QQ_{Y | \wh x_c} \in \mbb B_{Y | \wh x_c},~\mu = \EE_{\QQ_{Y|\wh x_c}} [T(Y)] \right\}.
\]
\begin{lemma}[Compactness of expectation parameter uncertainty set] \label{lemma:compact}
   The set $\mc S_c$ is compact, and it has an interior point whenever $\rho_c > 0$.
\end{lemma}
\begin{proof}[Proof of Lemma~\ref{lemma:compact}]
    By Lemma~\ref{lemma:KL-exp} and the definition of the set $\mc S_c$, we can write $\mc S_c$ as
    \[
        \mc S_c = \left\{ \mu \in \dom(\phi): \phi(\mu) - \phi(\wh \mu_c) - \inner{\mu - \wh \mu_c}{\wh \theta_c} \leq \rho_c \right\}.
    \]
    Because $\phi$ is closed, convex, proper, and that $\wh \theta_c \in \inte(\Theta) = \Theta$, the function $\phi(\cdot) - \inner{\cdot}{\wh \theta_c}$ is coercive by~\cite[Corollary~14.2.2]{ref:rockafellar1970convex} and \cite[Fact~2.11]{ref:bauschke1997legendre}. As a consequence, $\mc S_c$ is bounded.

    Because $\Psi$ is essentially strictly convex on $\Theta$, $\phi$ is essentially smooth on $\inte(\dom(\phi))$ by~\cite[Theorem~26.3]{ref:rockafellar1970convex}. \cite[Theorem~3.8]{ref:bauschke1997legendre} now implies that if $\mu'$ is a boundary point of $\inte(\dom(\phi))$ then as $\inte(\dom(\phi)) \ni \mu_k \xrightarrow{k \to \infty} \mu'$ then $\phi(\mu_k) - \inner{\mu_k}{\wh \theta_c} \xrightarrow{k \to \infty} + \infty$. Moreover, because $\phi$ is continuous over $\inte(\dom(\phi))$, the set $\mc S_c$ is closed. This implies that $\mc S_c$, being a closed and bounded set of finite dimension, is compact.
    
    The continuity of $\phi$ leads a straightforward manner to the non-empty interior of $\mc S_c$ when $\rho_c > 0$. This observation completes the proof.
\end{proof}
\begin{proof}[Proof of Proposition~\ref{prop:conditional-refor}]
	Because $\lambda$ is a mapping onto the space $\Theta$ of natural parameters, we use the shorthand $\lambda_c = \lambda(w, \wh x_c) \in \Theta$. Moreover, let $\wh \mu_c = \nabla \Psi(\wh \theta_c)$. The worst-case conditional expectation of the log-loss function becomes
	\begin{subequations}
		\begin{align*}
			\Sup{\QQ_{Y|\wh x_c} \in \mbb B_{Y|\wh x_c}}  \EE_{\QQ_{Y|\wh x_c}} \left[ \ell_\lambda(\wh x_c, Y, w) \right] &= 	\Sup{\QQ_{Y|\wh x_c} \in \mbb B_{Y|\wh x_c}}  \EE_{\QQ_{Y|\wh x_c}} \left[ \Psi( \lambda(w, \wh x_c)) - \inner{T(Y)}{\lambda(w, \wh x_c)} \right]\\
			 &= 	\Sup{\QQ_{Y|\wh x_c} \in \mbb B_{Y|\wh x_c}} ~ \Psi( \lambda(w, \wh x_c)) - \inner{\EE_{\QQ_{Y|\wh x_c}} [T(Y)]}{\lambda(w, \wh x_c)} \\
			&= \left\{
				\begin{array}{cl}
					\sup &  \Psi(\lambda_c) -\inner{\mu}{\lambda_c} \\
					\st & \phi(\mu) - \phi(\wh \mu_c) - \inner{\mu - \wh \mu_c}{\wh \theta_c} \le \rho_c,
				\end{array}
			\right. 
		\end{align*}
	\end{subequations}
	where the first equality is from the definition of $\ell_{\lambda}$ and the second equality follows from the linearity of the expectation operator. The last equality follows from the definition of the ambiguity set $\mbb B_{Y|\wh x_c}$ using the $\phi$ function by Lemma~\ref{lemma:KL-exp}.
	Because the term $\Psi(\lambda_c)$ does not involve the decision variable $\mu$, it suffices now to consider the optimization problem
	\be \label{eq:inner1}
		\sup \left\{\inner{-\lambda_c}{\mu}~ : ~\phi(\mu)  - \inner{\mu }{\wh \theta_c} \le \rho_c + \phi(\wh \mu_c) - \inner{\wh \mu_c}{\wh \theta_c} \right\}.
	\ee	
	Suppose at this moment that $\lambda_c \neq 0$ and $\rho_c > 0$. When $\rho_c > 0$, the feasible set of~\eqref{eq:inner1} satisfies the Slater condition because $\phi$ is a continuous function. Hence, by a strong duality argument, the convex optimization problem~\eqref{eq:inner1} is equivalent to
	\begin{align*}
		\Sup{\mu} \Inf{\dualvar \ge 0}~ \inner{-\lambda_c}{\mu} + \dualvar (\bar \rho_c -  \phi(\mu)  + \inner{\mu }{\wh \theta_c} ) =  \Inf{\dualvar \ge 0}~ \left\{\dualvar \bar \rho_c + \Sup{\mu} \inner{\mu}{ \dualvar \wh \theta_c - \lambda_c} -  \dualvar \phi(\mu) \right\},
	\end{align*}
	where $\bar \rho_c \Let \rho_c + \phi(\wh \mu_c) - \inner{\wh \mu_c}{\wh \theta_c} \in \R$ and the interchange of the supremum and the infimum operators is justified thanks to~\cite[Proposition~5.3.1]{ref:bertsekas2009convex}. Consider now the infimum problem on the right hand side of the above equation. If $\dualvar = 0$, then the inner supremum subproblem on the right hand side is unbounded because $\lambda_c \neq 0$, thus $\dualvar = 0$ is never an optimal solution to the infimum problem. By utilizing the definition of the conjugate function, one thus deduce that problem~\eqref{eq:inner1} is equivalent to
	\be \label{eq:inner-equiv}
		\Inf{\dualvar > 0}~\dualvar \bar \rho_c + (\dualvar \phi)^*(\dualvar \wh \theta_c - \lambda_c) = \Inf{\dualvar > 0}~\dualvar \bar \rho_c + \dualvar \phi^*\Big(\wh \theta_c - \frac{\lambda_c}{\dualvar} \Big),
	\ee
	where the equality exploits the fact that $(\dualvar \phi)^*(\theta) = \dualvar \phi^*(\theta/\dualvar)$ for any $\dualvar > 0$ \cite[Table~3.2]{ref:borwein2006convex}. 
	
	We now show that the reformulation problem~\eqref{eq:inner-equiv} is valid when $\rho_c = 0$. Indeed, when $\rho_c = 0$, problem~\eqref{eq:inner1} has a unique feasible solution $\wh \mu_c$, thus its optimal value is $\inner{-\lambda_c}{\wh \mu_c}$. Moreover, in this case, problem~\eqref{eq:inner-equiv} becomes
	\begin{align*}
		&\Inf{\dualvar > 0}~ \dualvar \left[ \phi(\wh \mu_c) - \inner{\wh \mu_c}{\wh \theta_c} + \phi^*\left( \wh \theta_c - \frac{\lambda_c}{\dualvar} \right) \right] \\
		= &
	\inner{-\lambda_c}{\wh \mu_c} + \Inf{\dualvar > 0}~ \dualvar \left[ \phi(\wh \mu_c) - \inner{\wh \mu_c}{\wh \theta_c - \frac{\lambda_c}{\dualvar}} + \phi^*\left( \wh \theta_c - \frac{\lambda_c}{\dualvar} \right) \right].
	\end{align*}
	Notice that the term in the square bracket of the optimization problem on the right hand side is non-negative by the definition of the conjugate function. Thus, the infimum problem over $\dualvar$ admits the optimal value of 0 as $\dualvar$ tends to $+\infty$. As a consequence, when $\rho_c = 0$, both problem~\eqref{eq:inner1} and \eqref{eq:inner-equiv} have the same optimal value and they are equivalent.
	
	Consider now the situation where $\lambda_c = 0$. In this case, problem~\eqref{eq:inner-equiv} becomes
	\[
		\Inf{\dualvar > 0}~ \dualvar \rho_c + \dualvar \left( \phi(\wh \mu_c) - \inner{\wh \mu_c}{\wh \theta_c} + \phi^*(\wh \theta_c) \right).
	\]
	By definition of the conjugate function, we have $\phi^*(\wh \theta_c) \ge \inner{\wh \mu_c}{\wh \theta_c} - \phi(\wh \mu_c)$, and thus, by combining with the fact that $\rho_c \ge 0$, this infimum problem will admit the optimal value of 0. Notice that when $\lambda_c = 0$, the optimal value of problem~\eqref{eq:inner1} is also 0. This shows that~\eqref{eq:inner-equiv} is equivalent to~\eqref{eq:inner1} for any possible value of $\lambda_c$.
	Replacing $\phi\opt$ in~\eqref{eq:inner-equiv} by its equivalence $\Psi$ and substituting $\inner{\wh \mu_c}{\wh \theta_c} - \phi(\wh \mu_c)$ by its equivalence $\Psi(\wh \theta_c)$ complete the reformulation~\eqref{eq:inner-refor}.
\end{proof}

\begin{proof}[Proof of Theorem~\ref{thm:main}]
    By applying Proposition~\ref{prop:refor}, the distributionally robust MLE problem~\eqref{eq:dro} can be reformulated as
    \begin{align*}
    \Min{w \in \mc W} \Max{\QQ \in \mbb B(\Pnom)} \EE_{\QQ} \Big[ \ell_\lambda(X, Y, w )\Big] = 
    \left\{
    \begin{array}{cl}
		\inf & \alpha + \beta \eps + \beta \ds \sum_{c=1}^C \wh p_c \exp\left( \frac{t_c - \alpha}{\beta} - \rho_c - 1\right) \\
		\st & w \in \mc W, \; t \in \R^C,\; \alpha \in \R,\;\beta \in \R_{++}  \\
		& \Sup{\QQ_{Y|\wh x_c} \in \mbb B_{Y|\wh x_c}}  \EE_{\QQ_{Y|\wh x_c}} \left[ \ell_\lambda(\wh x_c, Y, w)\right] \leq t_c ~~~ \forall c =1, \ldots, C.
	\end{array}
	\right.
    \end{align*}
    Using Proposition~\ref{prop:conditional-refor} to reformulate each constraint of the above optimization problem leads to the desired result.
\end{proof}

%%%%%%%%%%%%%%%%%%%%%%%%%%%%%%%%%%
\section{Proofs of Section~\ref{sect:analysis}}
\label{sec:app3}
\begin{proof}[Proof of Proposition~\ref{prop:reweighting}]
%\viet{missing a sup over $q$ on the right hand side?}
Let $\mathbbm{1}$ denote the $N$ dimensional vector of all $1$'s. Let $\KL(q\parallel p)=\sum_{i=1}^N q_i\log(q_i/p_i)$, we have
\begin{align*}
\sup_{\QQ:\KL(\QQ\parallel\Pnom^{\text{emp}})\le \eps}\EE_{\QQ}[\ell_\lambda(X,Y,w)]& = \sup_{q:\KL(q\parallel\frac{1}{N}\mathbbm{1})\le\eps}\sum_{i=1}^N q_i\ell_{\lambda}(\wh x_i,\wh y_i,w)\\
& = \sup_{q:\KL(q\parallel\frac{1}{N}\mathbbm{1})\le\eps}\sum_{i=1}^N q_i \left(\Psi(\lambda(w,\wh x_i))-\inner{T(\wh y_i)}{\lambda(w,\wh x_i)}\right).
\end{align*}
On the other hand, we note
\begin{align*}
  \sup_{\QQ \in \mbb B(\Pnom)} \EE_{\QQ} \big[ \ell_\lambda(X, Y, w )\big] &=\sup_{q:\KL(q\parallel\frac{1}{N}\mathbbm{1})\le\eps}\sum_{i=1}^Nq_i\EE_{\QQ_{Y|\wh x_i}}\left[\ell_{\lambda}(\wh x_i,Y,w)\right]\\
  &=\sup_{q:\KL(q\parallel\frac{1}{N}\mathbbm{1})\le\eps}\sum_{i=1}^Nq_i\left(\Psi(\lambda(w,\wh x_i))-\inner{\nabla\Psi(\wh\theta_i)}{\lambda(w,\wh x_i)}\right)\\
  &=\sup_{q:\KL(q\parallel\frac{1}{N}\mathbbm{1})\le\eps}\sum_{i=1}^Nq_i\left(\Psi(\lambda(w,\wh x_i))-\inner{T(\wh y_i)}{\lambda(w,\wh x_i)}\right).
\end{align*}
Therefore the objective functions are the same and the two problems are equivalent.
\end{proof}

The proof of Proposition~\ref{prop:surrogate} relies on the following result.

\begin{lemma}
\label{lemma:regularization-out-nice}
Let $\Delta \subset \R^C$ be a  simplex and $\wh p \in \inte(\Delta)$ be a probability vector. For any two vectors $\wh t, t\opt \in \R^C$, any vector $\rho \in \R_+^C$ and any scalar $\eps \ge \wh p^\top \rho$, we have
    \begin{align*}
        &\sup\left\{
             q^\top t\opt -\wh p^\top \wh t:
            q \in \Delta,\ds \sum_{c =1}^C q_c (\log q_c - \log \wh p_c + \rho_c) \leq \eps  \right\} \\
        & \hspace{6cm} \le 
        \| t\opt -\wh t\|_{\infty}+\frac{\sqrt{2\eps}}{\min_{c}\sqrt{\wh p_c}}\sqrt{\sum_{c=1}^C \wh p_c (\wh t_c-\bar t)^2},
    \end{align*}
    where $\bar t=\wh p^\top\wh t$.
\end{lemma}

\begin{proof}[Proof of Lemma~\ref{lemma:regularization-out-nice}]
Let $\mathbbm{1}$ denote the $C$ dimensional vector of $1$'s, we have
    \begin{align*}
        &\left\{
        \begin{array}{cl}
            \sup &  q^\top t\opt -\wh p^\top \wh t\\
            \st  & q \in \Delta,~\ds \sum_{c =1}^C q_c (\log q_c - \log \wh p_c + \rho_c) \leq \eps 
        \end{array} \right. \\
        =& \left\{
        \begin{array}{cl}
            \sup &  q^\top (t\opt -\wh t) +(q-\wh p)^\top \wh t \\
            \st  & q \in \Delta,~\ds \sum_{c =1}^C q_c (\log q_c - \log \wh p_c + \rho_c) \leq \eps 
        \end{array} \right. \\
        \le& \left\{
        \begin{array}{cl}
            \sup & q^\top (t\opt -\wh t) + (q-\wh p)^\top\wh t\\
            \st  & q \in \Delta,~\ds \sum_{c =1}^C (q_c - \wh p_c)^2 \leq 2\eps 
        \end{array} \right. \\
        \le&  \Sup{ \|q\|_1 = 1}q^\top (t\opt - \wh t) + \sup\left\{
            (q - \wh p)^\top (\wh t-\bar t\mathbbm{1}) :~ \| q - \wh p \|_2^2 \leq 2\eps \right\} \\
        \le &  \Sup{ \|q\|_1 = 1}q^\top(t\opt -\wh t) + \sup\left\{\sum_{c=1}^C \frac{q_c-\wh p_c}{\sqrt{\wh p_c}}\sqrt{\wh p_c}(\wh t_c-\bar t):~ \| q - \wh p \|_2^2 \leq 2\eps \right\} \\
        \le &  \Sup{ \|q\|_1 = 1}q^\top (t\opt - \wh t) + \frac{\sqrt{2\eps}}{\min_{c}\sqrt{\wh p_c}}\sqrt{\sum_{c=1}^C \wh p_c (\wh t_c-\bar t)^2},
    \end{align*}
    where the first inequality follows from Pinsker's inequality~\cite[Theorem 4.19]{ref:boucheron2013concentration} and the fact that $\|q-\wh p\|^2_2\leq\|q-\wh p\|^2_1=4\|q-\wh p\|^2_{TV}$, the second inequality follows from the fact that $(q-\wh p)^\top \mathbbm{1}=0$ and dropping the constraint $q \in \Delta$, and the last inequality is from Cauchy-Schwarz.
    
    In the last step, we have
    \begin{align*}
        \Sup{ \|q\|_1 = 1} q^\top (t\opt - \wh t) = \| t\opt - \wh t \|_{\infty},
    \end{align*}
    which completes the proof.
\end{proof}

We now ready to prove Proposition~\ref{prop:surrogate}.

\begin{proof}[Proof of Proposition~\ref{prop:surrogate}]
    Let $t\opt$ and $\wh t$ be two $C$-dimensional vectors whose elements are defined as
    \[
        t_c\opt = \Sup{\QQ_{Y|\wh x_c} \in \mbb B_{Y|\wh x_c}}  \EE_{\QQ_{Y|\wh x_c}} \left[\ell_\lambda(\wh x_c, Y, w ) \right], \quad 
        \wh t_c =   \EE_{\Pnom_{Y|\wh x_c}} \left[\ell_\lambda(\wh x_c, Y, w ) \right] \qquad \forall c.
    \]
    By Lemma~\ref{lemma:regularization-out-nice}, we find
    \begin{align*}
        \Sup{\QQ \in \mbb B(\Pnom)} \EE_{\QQ}[\ell_\lambda(X, Y, w)] - \EE_{\Pnom}[\ell_\lambda(X, Y, w)] &= \left\{
            \begin{array}{cl}
            \sup & q^\top t\opt -\wh p^\top \wh t \\
            \st & q \in \Delta,~\ds \sum_{c =1}^C q_c (\log q_c - \log \wh p_c + \rho_c) \leq \eps  
            \end{array}
            \right. \\
        \le & \| t\opt -\wh t\|_{\infty}+\frac{\sqrt{2\eps}}{\min_{c}\sqrt{\wh p_c}}\sqrt{\sum_{c=1}^C \wh p_c (\wh t_c-\bar t)^2},
    \end{align*}
    where $\bar t=\wh p^\top\wh t$.
    In the last step, notice that
    \[
        \sum_{c=1}^C\wh p_c(\wh t_c-\bar t)^2=\mathrm{Var}_{\Pnom_{X}}\left(\EE_{ \Pnom_{Y|X}}\left[\ell_{\lambda}(X,Y,w)\right]\right)\le \mathrm{Var}_{\Pnom}\left(\ell_\lambda(X, Y, w ) \right).
    \]
    It now remains to provide the bounds for $\|t\opt - \wh t\|_{\infty}$. For any $c$, let $\lambda_c = \lambda(w, \wh x_c)$, we have
    \begin{align*}
        t_c\opt - \wh t_c &= \left\{
            \begin{array}{cl}
                \sup & \inner{\mu - \wh \mu_c}{\lambda_c} \\
                \st & \phi(\mu) - \phi(\wh \mu_c) - \inner{\mu - \wh \mu_c}{\wh \theta_c} \le \rho_c.
            \end{array}
        \right. 
    \end{align*}
    Because $\Psi$ has locally Lipschitz continuous gradients, $\phi$ is locally strongly convex~\cite[Theorem~4.1]{ref:goebel2008local}. Moreover, the feasible set $\mc S_c$ of the above problem is compact by Lemma~\ref{lemma:compact}, hence there exists a constant $0 < m_c$ such that
    \[
    \frac{m_c}{2} \| \mu - \wh \mu_c\|_2^2 \le \phi(\mu) - \phi(\wh \mu_c) - \inner{\mu - \wh \mu_c}{\wh \theta_c} \quad \forall \mu \in \mc S_c.
    \]
    Notice that the constants $m_c$ depends only on $\Psi$ and $\wh \theta_c$. Thus, we find
     \begin{align*}
        t_c\opt - \wh t_c \le \sup \left\{
             \inner{\mu - \wh \mu_c}{\lambda_c} : m_c\| \mu - \wh \mu_c\|_2^2 \le 2\rho_c
        \right\}
       = \sqrt{2\rho_c/m_c}\| \lambda(w, \wh x_c) \|_2.
    \end{align*}
    By setting $m = \min_{c} m_c$, we have
    \[
        \| t\opt - \wh t\|_{\infty} \leq  \sqrt{\frac{2\max_{c}\rho_c}{m}} \| \lambda(w, \wh x_c) \|_2.
    \]
    Combining terms leads to the postulated results.
\end{proof}

For any $\wh \theta_c \in \Theta$, $\rho_c \in \R_+$, let $\mc R_{\wh \theta_c, \rho_c}(w)$ denote the value of the worst-case expected log-loss
\[
\mc R_{\wh \theta_c, \rho_c}(w) = \Sup{\QQ_{Y|\wh x_c} \in \mbb B_{Y|\wh x_c}}  \EE_{\QQ_{Y|\wh x_c}} \left[\ell_\lambda(\wh x_c, Y, w ) \right].
\]

\begin{lemma} \label{lemma:regularization-inner}
    %Suppose that the log-partition function $\Psi$ has locally Lipschitz continuous gradients and that $\Theta_c \subset \Theta$ is a compact set. Then there exist constants $0 < m < M < +\infty$ such that for any parameter mapping $\lambda$, any covariate $\wh x_c$, any weight $w$ and any radius $\rho_c \ge 0$, we have
    Suppose that the log-partition function $\Psi$ has locally Lipschitz continuous gradients, that $\Theta = \R^p$ and that $\Theta_c \subset \Theta$ is a compact set. For any fixed $\overline{\rho}_c \in \R_{++}$, there exist  constants $0 < m < M < +\infty$ that depend only on $\Psi$, $\Theta_c$ and $\overline{\rho}_c$ such that for any value $\lambda(w, \wh x_c) \in \R^p$ and any radius $\overline{\rho}_c \ge \rho_c \ge 0$
    \[
        \sqrt{2\rho_c/M}\| \lambda(w, \wh x_c) \|_2 \le \mc R_{\wh \theta_c, \rho_c}(w) - \mc R_{\wh \theta_c, 0}(w) \le \sqrt{2\rho_c/m}\| \lambda(w, \wh x_c) \|_2 \qquad \forall \wh \theta_c \in \Theta_c.
    \]
\end{lemma}
\begin{proof}[Proof of Lemma~\ref{lemma:regularization-inner}]
    Consider the set 
    \[
        \mc D \Let \{ \wh \mu_c : \exists \wh \theta_c \in \Theta_c \text{ such that } \wh \mu_c = \nabla \Psi(\wh \theta_c) \}
    \] 
    and its $\overline{\rho}_c$-inflated set
    \[
        \mc D_{\overline{\rho}_c} \Let \{ \mu: \exists \wh \mu_c \in \mc D \text{ such that }  \phi(\mu) - \phi(\wh \mu_c) - \inner{\mu - \wh \mu_c}{\wh \theta_c} \le \overline{\rho}_c \}.
    \]
    Because $\Theta_c$ is compact and $\nabla \Psi$ is a continuous function, $\mc D$ is compact~\cite[Theorem~2.34]{ref:aliprantis06hitchhiker}. Note that we can rewrite $\mc D_{\overline{\rho}_c}$ as 
    \[
    \mc D_{\overline{\rho}_c} = \{ \mu: \exists \wh \mu_c \in \mc D \text{ such that }  \phi(\mu)+\inner{\mu }{-\wh \theta_c} \le \overline{\rho}_c + \phi(\wh \mu_c) -\inner{ \wh \mu_c}{\wh\theta_c}\}.
    \]
    Let $S$ be temporarily the set
    \[
    S = \left\{\mu:\phi(\mu)+\inf_{\wh\theta_c\in\Theta_c}\inner{\mu }{-\wh \theta_c} \le \overline{\rho}_c +\Sup{\wh\theta_c\in\Theta_c} \phi(\wh \mu_c) -\inner{ \wh \mu_c}{\wh\theta_c}<\infty\right\}.
    \]
    We have that $\mc D_{\overline{\rho}_c}\subseteq S$. 
    Recall the definition of $\phi$:
    \[
	\phi : \mu \mapsto \Sup{\theta \in \Theta} \left\{ \inner{\mu}{\theta} - \Psi(\theta) \right\}.
    \]
    Therefore $\phi(\cdot)$ is closed, convex and proper. Therefore by \cite[Proposition~2.16]{ref:bauschke1997legendre}, $\Theta=\mathbb{R}^p$ implies that $\phi(\cdot)$ is super-coercive, i.e., $\lim_{\|\mu\|_2\to\infty}\phi(\mu)/\|\mu\|_2\to\infty$. Thus
    \[
    \lim_{\|\mu\|_2\to\infty}\phi(\mu)+\inf_{\wh\theta_c\in\Theta_c}\inner{\mu}{-\wh\theta_c}\to\infty.
    \]
    Therefore $S$ is bounded, which implies that $\mc D_{\overline{\rho}_c}$ is also bounded.
    % \[
    % \phi(\mu) - \kappa \|\mu\| \le \rho_c +\Sup{\wh\theta_c\in\Theta_c} \phi(\wh \mu_c) -\inner{ \wh \mu_c}{\wh\theta_c}<\infty
    % \]
    %By the similar argument in Proof of Lemma 10.4, the function $\phi(\cdot)-\inner{\cdot}{\wh\theta_c}$ is coersive, thus 
    %$\phi(\cdot)-\Inf{\wh\theta_c\in\Theta_c}\inner{\cdot}{\wh\theta_c}$ coersive, 
    
%     otherwise, there exists $r\in \mathbb{R}$ and $\{\mu^k\}_{k\geq1}$ such that 
%     \[\phi(\mu^k)-\Inf{\wh\theta_c\in\Theta_c}\inner{\mu_k}{\wh\theta_c}\le r
%     \]
%     and $\mu^k\to\infty$. Thus we can find a sequence $\{\wh\theta_c^k\}_{k\geq 1}$ such that 
%     \[
%     \phi(\mu^k)-\inner{\mu_k}{\wh\theta_c^k}\le 2r
%   \]
     Since $\Theta_c$ is compact, there exists a subsequence $\{\wh\theta_c^{k_n}\}_{n\geq1}$ such that $\wh\theta_c^{k_n}\to\wh\theta_c^\infty\in\Theta_c$ as $n\to\infty$. 
    Since $\mc D_{\rho_c}$ is bounded, it suffices to show that $\mc D_{\rho_c}$ is closed. Choose any sequence $\{\mu^k\}_{k\geq 1}\in \mc D_{\rho_c}$ such that $\mu^k\to\mu^{\infty}$ as $k\to\infty$, we want to show that $\mu^{\infty}\in \mc D_{\rho_c}$. For each $k$, since $\mu^k\in \mc D_{\rho_c}$, there exists $\wh\mu_c^k\in\mc D$ and $\wh \theta_c^k\in \Theta_c$ such that $\phi(\mu^k) - \phi(\wh \mu_c^k) - \inner{\mu^k - \wh \mu_c^k}{\wh \theta_c^k} \le \rho_c$. Since $\mc D$ and $\Theta_c$ are compact, there exists a  subsequence $\{k_n\}_{n\geq 1}$ such that $\wh\mu_c^{k_{n}}\to\wh\mu_c^{\infty}$ and $\wh\theta_c^{k_{n}}\to\wh\theta_c^{\infty}$ for some $\wh\mu_c^{\infty}\in \mc D$ and $\wh\theta_c^{\infty}\in \Theta_c$. Since $\wh\mu_c^{k_n}=\nabla \Psi(\wh\theta_c^{k_n})$, by continuity we have $\wh\mu_c^\infty = \nabla \Psi( \wh\theta_c^\infty)$. Note that \[
    \phi(\mu^{k_n}) - \phi(\wh \mu_c^{k_n}) - \inner{\mu^{k_n} - \wh \mu_c^{k_n}}{\wh \theta_c^{k_n}} \le \rho_c,
    \]
    by continuity of $\phi$, we have
    \[
    \phi(\mu^{\infty}) - \phi(\wh \mu_c^{\infty}) - \inner{\mu^{\infty} - \wh \mu_c^{\infty}}{\wh \theta_c^{\infty}} \le \rho_c.
    \]
    Therefore $\mu^\infty\in \mc D_{\rho_c}$ and hence $\mc D_{\rho_c}$ is closed. 
    
    The finite dimensional set $\mc D_{\overline{\rho}_c}$ is closed and bounded, thus it is compact, and moreover $\mc D \subseteq \mc D_{\rho_c}$. The convex hull $\overline{\mc D}_{\overline{\rho}_c}$ of $\mc D_{\overline{\rho}_c}$ is also compact~\cite[Corollary~5.33]{ref:aliprantis06hitchhiker}. Because $\Psi$ has locally Lipschitz continuous gradients, $\phi$ is locally strongly convex~\cite[Theorem~4.1]{ref:goebel2008local}. Moreover, $\phi$ is also essentially smooth by Lemma~\ref{lemma:facts}\ref{fact:4}. Thus over the set $\overline{\mc D}_{\overline{\rho}_c}$, there exist constants $0 < m \le M < +\infty$ such that
    \[
    \frac{m}{2} \| \mu - \mu'\|_2^2 \le \phi(\mu) - \phi(\mu') - \inner{\mu - \mu'}{\theta'} \le \frac{M}{2} \| \mu - \mu'\|_2^2 \quad \forall \mu, \mu' \in \overline{\mc D}_{\overline{\rho}_c}, \mu' = \nabla \Psi(\theta').
    \]
    Notice that the constants $m$ and $M$ depend only on $\phi$, and thus on $\Psi$, $\overline{\rho}_c$ and $\Theta_c$
    
    Denote temporarily the shorthand $\lambda_c = \lambda(w, \wh x_c)$. We have
    $\mc R_{\wh \theta_c, 0}(w) = \Psi(\lambda_c)-\inner{\wh \mu_c}{\lambda_c}$,
    and so
    \begin{align*}
        \mc R_{\wh \theta_c, \rho_c}(w) - \mc R_{\wh \theta_c, 0}(w) &= \left\{
            \begin{array}{cl}
                \sup & \inner{\mu - \wh \mu_c}{\lambda_c} \\
                \st & \phi(\mu) - \phi(\wh \mu_c) - \inner{\mu - \wh \mu_c}{\wh \theta_c} \le \rho_c.
            \end{array}
        \right. 
    \end{align*}
    Because $\mu$ and $\wh \mu_c$ are both in $\overline{D}_{\overline{\rho}_c}$, we have
    \[
        \frac{m}{2} \| \mu - \wh \mu_c\|_2^2 \le \phi(\mu) - \phi(\wh \mu_c) - \inner{\mu - \wh \mu_c}{\wh \theta_c} \le \frac{M}{2} \| \mu - \wh \mu_c\|_2^2.
    \]
    We now have
    \[
        \mc R_{\wh \theta_c, \rho_c}(w) - \mc R_{\wh \theta_c, 0}(w) \leq \sup\left\{\inner{\mu - \wh \mu_c}{\lambda_c}: \| \mu - \wh \mu_c\|_2^2 \le 2\rho_c/m\right\} = \sqrt{2\rho_c/m}\| \lambda_c\|_2.
    \]
    A similar argument leads to the lower bound. This observation completes the proof.
\end{proof}

%\viet{consistency proof move here}
\begin{proof}[Proof of Theorem~\ref{thm:consistency}]
Without loss of generality consider $\mc W\subseteq\mathbb{R}^q$. For notational simplicity, denote
\[
    R_{\wh\theta,\eps,\rho}(w)=\Sup{\QQ \in \mbb B(\Pnom)} \EE_{\QQ}\left[  \ell_\lambda(X, Y, w ) \right]. 
\]
Since $\eps \geq \sum_{c=1}^C\wh p_c \rho_c$ with probability going to $1$, following the same argument as in the proof of Proposition \ref{prop:surrogate}, we have that with probability going to $1$, for any $w\in\mc W$,
\[
  R_{\wh\theta,\eps,\rho}(w) - R_{\wh\theta,0,0}(w)\le  \| t\opt - \wh t\|_1  + \sqrt{2\eps} \|\wh t\|_1,
\]
where
  \[
        \| \wh t \|_1 = \sum_{c=1}^C | \EE_{\Pnom_{Y|\wh x_c}} \left[\ell_\lambda(\wh x_c, Y, w ) \right] |
    \quad \text{and} \quad
   \| t\opt - \wh t\|_1  = \sum_{c=1}^C |\mc R_{\wh \theta_c, \rho_c}(w) - \mc R_{\wh \theta_c, 0}(w)|.
\]
For each $w$, since $\wh \theta_c\to \lambda(w_0,\wh x_c)$ in probability, we have $\PP(\|\wh \theta_c-\lambda(w_0,\wh x_c)\|_2>1)\to0$. %\viet{which norm is this?}.
Therefore there exists compact set $\Theta_c$ for each $c$ such that $\wh \theta_c$ is contained in $\Theta_c$ with probability going to $1$. %\viet{should we say something like: As the number of samples gets sufficiently big, $\wh \theta_c$ belongs to a compact set $\Theta_c$ for all $c$ with probability 1?I don't think with probability 1 can be said, we can have uniform control on the required sample size only with proability going to 1}. 
Choose $\overline{\rho}_c=1$, since $\rho_c\to0$, we have $\overline{\rho}_c\geq \rho_c$ eventually. Therefore, by Lemma~\ref{lemma:regularization-inner}, for each $c$ with probability going to $1$
\[
 |\mc R_{\wh \theta_c, \rho_c}(w) - \mc R_{\wh \theta_c, 0}(w)|\le\sqrt{2\rho_c/m}\| \lambda(w, \wh x_c) \|_2,
\] 
where the above constant $m$ can be chosen independent of $c$ due to the finite cardinality assumption of $\mathcal{X}$. 
Since the function $\lambda(w,\wh x_c)$ is continuous in $w$ for any $\wh x_c$, we have $\|\lambda(w,\wh x_c)\|_2$ is bounded for all $w$ ranging over a compact set $W \subset \mc W$. Thus for each $c$ with probability going to $1$, we have
\[
 \sup_{w\in W}|\mc R_{\wh \theta_c, \rho_c}(w) - \mc R_{\wh \theta_c, 0}(w)|\le \sqrt{2\rho_c/m}\sup_{w\in W}\|\lambda(w,\wh x_c)\|_2.
\]
Since $\rho_c\to0$, we have for each $c$
\[
 \sup_{w\in W}|\mc R_{\wh \theta_c, \rho_c}(w) - \mc R_{\wh \theta_c, 0}(w)|=o_{\PP}(1).
\]
Thus $\sup_{w\in W}\|t\opt -\wh t\|_1=o_{\PP}(1)$. On the other hand, since $\sup_{w\in W}\mathcal{R}_{\wh\theta_c,0}(w)$ is $O_{\PP}(1)$, we have $\sup_{w\in W}\|\wh t\|_1=O_{\PP}(1)$. Therefore as $\eps\to0,\rho_c\to0$,
\[
 \sup_{w\in W}|R_{\wh \theta,\eps,\rho}(w)-R_{\wh\theta,0,0}(w)|=o_{\PP}(1)
\]
for any compact set $W$. Next, since $\wh \theta_c\to \lambda(w_0,\wh x_c)$ in probability, we have by continuous mapping theorem 
\[
    \nabla \Psi(\wh \theta_c)\to\nabla \Psi(\lambda(w_0,\wh x_c))\text{ in probability}.
\]
Besides, by the strong law of large number, 
\[ \wh p_c\to \mathbb{P}(X=\wh x_c)\text{ almost surely.}
\]
Recall that%\viet{equation overflow, please cut them down into multiple lines}
\begin{align*}
     R_{\wh \theta,0,0}(w)&= \EE_{\Pnom}[\ell_{\lambda}(X,Y,w)] = \sum_{c=1}^C\wh p_c\EE_{\Pnom_{Y|\wh x_c}}[\ell_{\lambda}(\wh x_c,Y,w)]\\
     &=\sum_{c=1}^C \wh p_c \left(\Psi(\lambda(w,\wh x_c))-\langle\nabla \Psi(\wh \theta_c),\lambda(w,\wh x_c)\rangle\right).
\end{align*}
Therefore, for each $w$, we have
\[ 
    R_{\wh\theta,0,0}(w)\to R(w)\text{ in probability},
\]
where
\[
R(w) = \EE_{\PP}[ \ell_{\lambda}(X, Y, w)] =\sum_{c=1}^C\mathbb{P}(X=\wh x_c)\left(\Psi(\lambda(w,\wh x_c))-\inner{\nabla \Psi(\lambda(w_0,\wh x_c))}{\lambda(w,\wh x_c)}\right).
\]
Since for each $c$, 
\[
  w_0 = \min_{w\in\mc W} \Psi(\lambda(w,\wh x_c))-\inner{\nabla \Psi(\lambda(w_0,\wh x_c))}{\lambda(w,\wh x_c)}
\]
Therefore $w_0$ solves $\min_{w\in\mc W} R(w)$. If $R(w)$ admits an unique solution, then clearly $w_0$ is such a solution. Since $R_{\wh \theta,0,0}(\cdot)$ is convex, by \cite[Theorem~II.1]{ref:andersen1982cox},
\[
  \sup_{w\in W}|R_{\wh\theta,0,0}(w)- R(w)| = o_{\PP}(1)
\]
for any compact set $W$.
Thus by triangle inequality
\[
   \sup_{w\in W}|R_{\wh \theta,\eps,\rho}(w)- R(w) |= o_{\PP}(1)
\]
for any compact set $W$.
Let $B$ denote the unit closed ball in $\mathbb{R}^q$, then $w_0+\eta B$ is compact for any $\eta>0$. Thus $ R_{\wh \theta,\eps,\rho}(w)- R(w) = o_{\PP}(1)$ uniformly over $w_0+\eta B$. Since $R(w)$ is convex and $w_0$ is its unique optimal solution, we have
\[
\inf_{w\in w_0+\eta B\backslash \frac{\eta}{2}B} R(w)>R(w_0).
\]
Therefore, with probability going to $1$, 
\[ 
   \inf_{w\in w_0+\frac{\eta}{2} B} R_{\wh \theta,\eps,\rho}(w)< \inf_{w\in w_0+\eta B\backslash \frac{\eta}{2}B} R_{\wh \theta,\eps,\rho}(w).
\]
Thus by convexity of $R_{\wh\theta,\eps,\rho}$, also
\[ 
   \inf_{w\in w_0+\frac{\eta}{2} B} R_{\wh \theta,\eps,\rho}(w)< \inf_{w\notin w_0+\eta B} R_{\wh \theta,\eps,\rho}(w).
\]
Thus the solution $w^*$ that solves $\inf_{w\in 
  \mathcal{W}} R_{\wh \theta,\eps,\rho}(w)$
satisfies 
\[
   \PP(\|w^*-w_0\|_2\leq \frac{\eta}{2})\to1.
\]
Since $\eta$ is chosen arbitrarily, we conclude that $w^*\to w_0$ in probability.
\end{proof}

\begin{proof}[Proof of Lemma~\ref{lemma:asymptotic-joint}]
Denote
\[
    W_c  =  \sqrt{N_c} \left(\frac{\sum_{\wh x_i =\wh x_c} T(\wh y_i)}{N_c} - \EE_{f(\cdot|\theta_c)}[T(Y)]\right).
\]
W.l.o.g. we can assume that $\EE_{f(\cdot|\theta_c)}[T(Y)] =0 $. 
We first show the joint convergence
\[
    (W_1^\top,\ldots, W_C^\top)^\top \xrightarrow{d.} \mc N(0,G)\qquad \text{as}\qquad N\to\infty,
\]
where $G$ is a block-diagonal matrix with diagonal blocks given by $G_c =\mathrm{Cov}_{f(\cdot|\theta_c)}(T(Y)),c=1,\ldots,C$. Note that 
\[
    N_c/N \to \PP(X = \wh x_c)>0\qquad \text{a.s. for each }c.
\]
For convenience denote $r_c = \PP(X = \wh x_c)$. We let 
\[
    \tilde {W_c} =  \sqrt{\lfloor r_c N\rfloor}\cdot \frac{\sum_{\wh x_i =\wh x_c} T(\wh y_i)}{\lfloor r_c N\rfloor} =  \frac{\sum_{\wh x_i =\wh x_c} T(\wh y_i)}{\sqrt{\lfloor r_c N\rfloor}}.
\]
Let $\left[\sum_{\wh x_i =\wh x_c} T(\wh y_i)\right]_{\lfloor r_c N\rfloor}$ be the sum of the first $\lfloor r_c N\rfloor$ samples of $T(\wh y_i)$ such that $\wh x_i = \wh x_c$. 
If $N_c < \lfloor r_c N\rfloor$, we add additional $\lfloor r_c N\rfloor - N_c$ independent copies of $T(Y)$ where $Y\sim f(\cdot|\theta_c)$ to the sum $\sum_{\wh x_i =\wh x_c} T(\wh y_i)$, and denote it by $\left[\sum_{\wh x_i =\wh x_c} T(\wh y_i)\right]_{\lfloor r_c N\rfloor}$ as well. Denote
\[
  \bar {W_c} = \frac{\left[\sum_{\wh x_i =\wh x_c} T(\wh y_i)\right]_{\lfloor r_c N\rfloor}}{\sqrt{\lfloor r_c N\rfloor} }.
\]
Note that $\bar W_1,\ldots,\bar W_C$ are independent, by i.i.d~ central limit theorem
\[
\left(\bar W_1^\top, \ldots, \bar W_C^\top \right)^\top\xrightarrow{d.} \mc N(0,G)\qquad \text{as}\qquad N\to\infty,
\]
where $G$ is a block-diagonal matrix with $G_c =\mathrm{Cov}_{f(\cdot|\theta_c)}(T(Y))$. We next show that 
 \[
 \tilde W_c - \bar W_c = o_{\PP}(1).
 \]
Note that
 \[
 \tilde W_c - \bar W_c = \frac{\left[\sum_{\wh x_i =\wh x_c} T(\wh y_i)\right]_{\lfloor r_c N\rfloor} - \sum_{\wh x_i =\wh x_c} T(\wh y_i)}{\sqrt{\lfloor r_c N\rfloor}}.
 \]
 By Chebyshev inequality
 \begin{align*}
 \PP(\|\tilde W_c - \bar W_c \|_2>\epsilon) & \leq \frac{\EE\left[\left\|\left[\sum_{\wh x_i =\wh x_c} T(\wh y_i)\right]_{\lfloor r_c N\rfloor} - \sum_{\wh x_i =\wh x_c} T(\wh y_i)\right\|_2^2\right]}{\epsilon^2 \lfloor r_c N\rfloor}\\
 &= \frac{\EE\left[\EE\left[\left\|\left[\sum_{\wh x_i =\wh x_c} T(\wh y_i)\right]_{\lfloor r_c N\rfloor} - \sum_{\wh x_i =\wh x_c} T(\wh y_i)\right\|_2^2\right]\bigg|N_c\right]}{\epsilon^2 \lfloor r_c N\rfloor}\\
 & =  \frac{\EE[\|T(\wh y_i)\|_2^2]}{\epsilon^2}\frac{\EE[|\lfloor r_c N\rfloor - N_c|]}{\lfloor r_c N\rfloor}.
 \end{align*}
 Since $N_c/\lfloor r_c N\rfloor\to 1$ almost surely, by dominated convergence theorem
 \[
 \frac{\EE[|\lfloor r_c N\rfloor - N_c|]}{\lfloor r_c N\rfloor}\to 0.
 \]
 Thus
 \[
 \PP(\|\tilde W_c - \bar W_c \|_2>\epsilon)\to 0 \qquad \text{as}\qquad N\to\infty,
 \]
 which means that $\tilde W_c - \bar W_c = o_{\PP}(1)$. Thus by Slutsky's lemma 
 \[
\left(\tilde W_1^\top, \ldots, \tilde W_C^\top \right)^\top\xrightarrow{d.} \mc N(0,G)\qquad \text{as}\qquad N\to\infty.
\]
Finally, since $W_c = (1+o_{\PP}(1))\tilde W_c$, by Slutsky's lemma,
\[
\left( W_1^\top, \ldots, W_C^\top \right)^\top\xrightarrow{d.} \mc N(0,G)\qquad \text{as}\qquad N\to\infty.
\]

 Now note that
 \[
\wh\theta_c =(\nabla \Psi)^{-1}\left((N_c)^{-1}\sum_{\wh x_i=\wh x_c}T(\wh y_i)\right) 
\]
and 
\[
\theta_c = (\nabla \Psi)^{-1}\left(\EE_{f(\cdot|\theta_c)}[T(Y)]\right).
\]
Also note that the vector-valued function $(\nabla\Psi)^{-1}(\cdot)$ is continuously differentiable at $\EE_{f(\cdot|\theta_c)}[T(Y)]$, therefore, by the delta method
\[
\left(\sqrt{N_1}(\wh \theta_1-\theta_1)^\top, \ldots, \sqrt{N_C}(\wh \theta_C-\theta_C)^\top\right)^\top \xrightarrow{d.} D\cdot \mc N(0,G),
\]
where $D$ is a block-diagonal matrix with diagonal elements given by
\[
 D_c= J (\nabla \Psi)^{-1}( \EE_{f(\cdot|\theta_c)}[T(Y)])
\]
the Jacobian matrix of $(\nabla \Psi)^{-1}$ evaluated at $\EE_{f(\cdot|\theta_c)}[T(Y)]$. Thus
\[
V_c =  D_c \mathrm{Cov}_{f(\cdot|\theta_c)}(T(Y))D_c^\top.
\]
Note that by Lemma~\ref{lemma:KL-exp}, we find
\[
     \KL(f(\cdot | \theta_c) \parallel f(\cdot | \wh \theta_c))=\inner{ \theta_c-\wh \theta_c}{ \mu_c}+\Psi(\wh \theta_c)-\Psi(\theta_c).
\]
Note that $\Psi$ is infinitely-many differentiable, we have the follow Taylor expansion
\begin{align*}
    &\Psi(\wh \theta_c)-\Psi(\theta_c)=\inner{\wh \theta_c-\theta_c}{\mu_c}+\frac{1}{2}\inner{\wh \theta_c-\theta_c}{\nabla^2\Psi\big(\theta_c+\eta(\wh \theta_c-\theta_c)\big)(\wh \theta_c-\theta_c)},
\end{align*}
where $\eta$ is a random variable with values between $0$ and $1$. 
Therefore
\[ 
    \KL(f(\cdot | \theta_c) \parallel f(\cdot | \wh \theta_c))=\frac{1}{2}\inner{\wh \theta_c-\theta_c}{\nabla^2\Psi\big(\theta_c+\eta(\wh \theta_c-\theta_c)\big)(\wh \theta_c-\theta_c)}.
\]
Because $\sqrt{N_c} (\wh \theta_c - \theta_c) \xrightarrow{d.} \mc N(0, V_c)$, and $\nabla^2\Psi(\cdot)$ is continuous, we have
\[
    \nabla^2\Psi\big(\theta_c+\eta(\wh \theta_c-\theta_c)\big)=\nabla^2\Psi\left(\theta_c\right)+o_{\PP}(1).
\]
Moreover, since we have the joint convergence 
\[
    \left(\sqrt{N_1}(\wh\theta_1-\theta_1)^\top,\ldots,\sqrt{N_C}(\wh\theta_C-\theta_C)^\top\right)^\top \xrightarrow{d.} \mc N(0,V),
\]
by continuous mapping theorem
\[
    \left(N_1\times\KL(f(\cdot|\theta_1)\parallel f(\cdot|\wh\theta_1)),\ldots, N_C\times \KL(f(\cdot|\theta_C)\parallel f(\cdot|\wh\theta_C))\right)^\top \xrightarrow{d.}Z \quad \text{as}\quad N\to\infty,
\]
 where $Z = (Z_1,\ldots, Z_C)^\top$ with $Z_c = \frac{1}{2}R_c^\top\nabla^2\Psi(\theta_c)R_c$, $R_c\sim\mc N(0,V_c)$ and are independent for $c=1,\ldots,C$.
\end{proof}

Before proving the result on the worst-case distribution in Theorem~\ref{thm:extreme}, we first prove the worst-case conditional measure that maximize problem~\eqref{eq:inner}. 
\begin{proposition}[Worst-case conditional distribution] \label{prop:conditional-extreme}
    For any $w \in \mc W$ and $\rho_c \in \R_{++}$, then the supremum problem~\eqref{eq:inner} is attained by $\QQ_{Y|\wh x_c}\opt \sim f(\cdot | \theta_c\opt)$ with $\theta_c\opt = \wh \theta_c - \lambda(w, \wh x_c)/\dualvar_c\opt$, where $\dualvar_c\opt > 0$ is the solution of the nonlinear algebraic equation
    \be \label{eq:FOC}
        \Psi\big( \wh \theta_c - \dualvar^{-1}\lambda(w, \wh x_c) \big) + \dualvar^{-1} \inner{\nabla \Psi\big(\wh \theta_c - \dualvar^{-1} \lambda(w, \wh x_c) \big)}{\lambda(w, \wh x_c)} = \Psi(\wh \theta_c) - \rho_c.
    \ee
\end{proposition}
\begin{proof}[Proof of Proposition~\ref{prop:conditional-extreme}]
    Reminding that problem~\eqref{eq:inner} is written as
    \[
    \Sup{\QQ_{Y|\wh x_c} \in \mbb B_{Y|\wh x_c}}  \EE_{\QQ_{Y|\wh x_c}} \left[\ell_\lambda(\wh x_c, Y, w ) \right].
    \]
    In the first step, we show that $\QQ_{Y|\wh x_c}\opt$ is feasible in problem~\eqref{eq:inner}, which means that $\QQ_{Y|\wh x_c}\opt \in \mbb B_{Y|\wh x_c}$. Indeed, we find that
    \[
        \KL(\QQ_{Y|\wh x_c}\opt \parallel \Pnom_{Y|\wh x_c}) = -\Psi\Big( \wh \theta_c - \frac{\lambda(w, \wh x_c)}{\dualvar_c\opt} \Big) - \frac{1}{\dualvar_c\opt} \inner{\nabla \Psi\Big(\wh \theta_c - \frac{\lambda(w, \wh x_c)}{\dualvar_c\opt} \Big)}{\lambda(w, \wh x_c)} + \Psi(\wh \theta_c) = \rho_c,
    \]
    where the first equality exploits the expression of the KL divergence between two distributions from the same family in Lemma~\ref{lemma:KL-exp}, and the second equality follows from the fact that $\dualvar_c\opt$ solves~\eqref{eq:FOC}.
    
    Proposition~\ref{prop:conditional-refor} asserts that the worst-case conditional expected log-loss problem~\eqref{eq:inner} is equivalent to the convex program~\eqref{eq:inner-refor}. Noticing that~\eqref{eq:FOC} is the first-order optimality condition of problem~\eqref{eq:inner-refor}, thus, by definition, $\dualvar_c\opt$ is the minimizer of~\eqref{eq:inner-refor}. The objective value of $\QQ_{Y|\wh x_c}\opt$ in~\eqref{eq:inner} amounts to
    \begin{align*}
        \EE_{\QQ_{Y|\wh x_c}\opt} \left[\ell_{\lambda}(\wh x_c, Y, w ) \right] &= \Psi( \lambda(w, \wh x_c)) - \inner{\EE_{\QQ_{Y|\wh x_c}\opt} [T(Y)]}{\lambda(w, \wh x_c)} \\
        &= \Psi( \lambda(w, \wh x_c)) - \inner{\nabla \Psi\Big( \wh \theta_c - \frac{\lambda(w, \wh x_c}{\dualvar_c\opt} \Big)}{\lambda(w, \wh x_c)} \\
        &= \dualvar_c\opt \big(\rho_c - \Psi(\wh \theta_c) \big) + \dualvar_c\opt \Psi\Big( \wh \theta_c - \frac{\lambda(w, \wh x_c)}{\dualvar_c\opt} \Big) + \Psi(\lambda(w, \wh x_c)),
    \end{align*}
    where the first equality follows by substituting the expression of $\ell_\lambda$ and the linearity of the expectation operator, the second equality follows from the convex conjugate relationship between the expectation parameters and the log-partition function $\Psi$, and the last equality follows from the fact that $\dualvar_c\opt$ solves~\eqref{eq:FOC}. Notice that the last expression coincide with the objective value of~\eqref{eq:inner-refor} evaluated at the optimal solution $\dualvar_c\opt$. This observation implies that $\QQ_{Y|\wh x_c}\opt$ attains the optimal value in~\eqref{eq:inner}.
\end{proof}

Next, we establish the following result on the optimal solution of the support function $h_{\mc Q}$ of the set $\mc Q$ defined as in~Lemma~\ref{lemma:support}.

\begin{lemma}[Support point of $\mc Q$] \label{lemma:support-solution}
Let $\mc Q$ be defined as in~\eqref{eq:Q-def}. For any $t \in \R^C$, if there exist $\alpha\opt \in \R$ and $\beta\opt \in \R_{++}$ that solve the following system of nonlinear algebraic equation
\begin{subequations}
    \begin{align}
            \ds \sum_{c=1}^C \wh p_c \exp \Big( \frac{t_c - \alpha}{\beta} - \rho_c - 1 \Big) - 1 &= 0 \label{eq:q-opt-1}\\
            \ds \sum_{c=1}^C \wh p_c (t_c - \alpha) \exp \Big( \frac{t_c - \alpha}{\beta} - \rho_c - 1 \Big) - (\eps + 1) \beta &= 0 \label{eq:q-opt-2}
        \end{align}
then the optimal solution $q\opt \in \mc Q$ that attains $t^\top q\opt = h_{\mc Q}(t)$ is
\be \label{eq:q-opt-def}
    q_c\opt = \wh p_c \exp\Big( \frac{t_c - \alpha\opt}{\beta\opt} - \rho_c - 1 \Big) \qquad \forall c = 1, \ldots, C.
\ee
\end{subequations}
\end{lemma}
\begin{proof}[Proof of Lemma~\ref{lemma:support-solution}]
    By definition of $q\opt$ in~\eqref{eq:q-opt-def}, one can verify that $q\opt \ge 0$ and that $\sum_{c = }^C q_c\opt = 1$, where the equality follows from~\eqref{eq:q-opt-1}. Moreover,
    \begin{align*}
    \ds \sum_{c =1}^C q_c\opt (\log q_c\opt - \log \wh p_c + \rho_c) &= \sum_{c =1}^C \wh p_c \Big( \frac{t_c - \alpha\opt}{\beta\opt} - 1 \Big) \exp\Big( \frac{t_c - \alpha\opt}{\beta\opt} - \rho_c - 1 \Big) \\
    &= \sum_{c =1}^C \wh p_c \Big( \frac{t_c - \alpha\opt}{\beta\opt} \Big) \exp\Big( \frac{t_c - \alpha\opt}{\beta\opt} - \rho_c - 1 \Big) - 1 = \eps,
    \end{align*}
    where the equalities follow from the definition of~$q\opt$ in~\eqref{eq:q-opt-def}, and the equations~\eqref{eq:q-opt-1} and~\eqref{eq:q-opt-2}, respectively. This implies that $q\opt \in \mc Q$.
    
    It now remains to show that $t^\top q\opt = h_{\mc Q}(t)$. By Lemma~\ref{lemma:support}, we have
	\[
		h_{\mc Q}(t) = \left\{
			\begin{array}{cl}
				\inf & \alpha + \eps \beta + \beta \sum_{c=1}^C \wh p_c \exp\Big( \frac{t_c - \alpha}{\beta} - \rho_c - 1 \Big) \\
				\st & \alpha \in \R,\; \beta \in \R_{++}.
			\end{array}
		\right.
	\]
	If $(\alpha\opt, \beta\opt) \in \R \times \R_{++}$ is the solution of~\eqref{eq:q-opt-1}-\eqref{eq:q-opt-2}, then $(\alpha\opt, \beta\opt)$ satisfy the Karush-Kuhn-Tucker condition of the above infimum optimization problem, and thus we have
	\[
		h_{\mc Q}(t) = \alpha\opt + \eps \beta\opt + \beta\opt \sum_{c=1}^C \wh p_c \exp\Big( \frac{t_c - \alpha\opt}{\beta\opt} - \rho_c - 1 \Big).
	\]
	Moreover, we find
	\begin{align*}
		\sum_{c=1}^C t_c q_c\opt &= \sum_{c=1}^C t_c \wh p_c \exp\Big( \frac{t_c - \alpha\opt}{\beta\opt} - \rho_c -1 \Big) \\
		&=(\eps + 1)\beta\opt + \alpha\opt \sum_{c=1}^C \wh p_c \exp \Big( \frac{t_c - \alpha\opt}{\beta\opt} - \rho_c - 1 \Big) \\
		&= \alpha\opt + \eps \beta\opt + \beta\opt \sum_{c=1}^C \wh p_c \exp \Big( \frac{t_c - \alpha\opt}{\beta\opt} - \rho_c - 1 \Big) = h_{\mc Q}(t),
	\end{align*}
	where the first equality follows from the definition of $q\opt$, the second equality follows from~\eqref{eq:q-opt-2} and the third equality follows from~\eqref{eq:q-opt-1}. This observation completes the proof.
\end{proof}

\begin{proof}[Proof of Theorem~\ref{thm:extreme}]
    It is easy to verify that $\QQ\opt$ is a probability measure because each $\delta_{\wh x_c}$ and $\QQ_{Y|\wh x_c}\opt$ is a probability measure, and $\sum_{c=1}^C \wh p_c \exp\big( (t_c\opt - \alpha\opt)/\beta\opt - \rho_c -1 \big) = 1$ since $\alpha\opt,\beta\opt$ solves
	\begin{align}
	\sum_{c=1}^C \wh p_c \exp \big( \beta^{-1}(t_c\opt - \alpha) - \rho_c - 1 \big) - 1 &= 0\label{eq:Q-opt-1} \\
	\sum_{c=1}^C \wh p_c (t_c\opt - \alpha) \exp \big( \beta^{-1}(t_c\opt - \alpha) - \rho_c - 1 \big) - (\eps + 1) \beta &= 0\label{eq:Q-opt-2}, 
	\end{align}
	If we set $\QQ_X\opt = \sum_{c=1}^C \wh p_c \exp\big( (t_c\opt - \alpha\opt)/\beta\opt - \rho_c -1 \big) \delta_{\wh x_c}$, then we have
    \[
		\QQ\opt (\{\wh x_c\} \times A) = \QQ_X\opt(\{\wh x_c\}) \QQ_{Y|\wh x_c}\opt(A) \quad \forall A \in \mc F(\mc Y),~\forall c.
    \]
    Moreover, because $\QQ_{Y|\wh x_c}\opt$ is constructed using Proposition~\ref{prop:conditional-extreme}, we have $\KL(\QQ_{Y|\wh x_c} \parallel \Pnom_{Y|\wh x_c} ) \le \rho_c$ for all $c$. Furthermore, we also have
    \begin{align*}
        \KL(\QQ_X\opt \parallel \Pnom_X) + \EE_{\QQ_X\opt}[\sum_{c=1}^C \rho_c \mathbbm{1}_{\wh x_c}(X)]  &= \sum_{c =1}^C \wh p_c \Big( \frac{t_c\opt - \alpha\opt}{\beta\opt} - 1 \Big) \exp\Big( \frac{t_c\opt - \alpha\opt}{\beta\opt} - \rho_c - 1 \Big) \\
        &= \sum_{c =1}^C \wh p_c \Big( \frac{t_c\opt - \alpha\opt}{\beta\opt} \Big) \exp\Big( \frac{t_c\opt - \alpha\opt}{\beta\opt} - \rho_c - 1 \Big) - 1 = \eps,
    \end{align*}
    where the equalities follow from the construction of $\QQ_X\opt$ and the equations~\eqref{eq:Q-opt-1} and~\eqref{eq:Q-opt-2}, respectively. This implies that $\QQ\opt \in \mbb B(\Pnom)$.

    It now remains to show that $\QQ\opt$ is optimal. For any weight $w$, by the definition of $t_c\opt$, we have
    \[
        t_c\opt = \EE_{\QQ_{Y|\wh x_c}\opt} \left[\ell_{\lambda}(\wh x_c, Y, w ) \right] = \Sup{\QQ_{Y|\wh x_c} \in \mbb B_{Y|\wh x_c}}  \EE_{\QQ_{Y|\wh x_c}} \left[ \ell_\lambda(\wh x_c, Y, w)\right]
    \]
    We thus find
    \begin{align}
        \Max{\QQ \in \mbb B(\Pnom)} \EE_{\QQ} \Big[ \ell_\lambda(X, Y, w )\Big] &= 
    	\Sup{\QQ_{X} \in \mbb B_{X}} \EE_{\QQ_{X}}\left[  \Sup{\QQ_{Y|X} \in \mbb B_{Y|X}} \EE_{\QQ_{Y|X}} \left[ \ell_\lambda(X, Y, w) \right] \right] \notag \\
    	&=\Sup{\QQ_{X} \in \mbb B_{X}} \EE_{\QQ_{X}}\left[ \sum_{c=1}^C t_c\opt \mathbbm{1}_{\wh x_c}(X)\right] \notag\\
    	&= \Sup{q \in \mc Q}~q^\top t\opt \label{eq:Q-opt-12}\\
    	&= \sum_{c=1}^C \wh p_c t_c\opt \exp\Big( \frac{t_c\opt - \alpha\opt}{\beta\opt} - \rho_c -1 \Big) \label{eq:Q-opt-13}\\
    	&= \EE_{\QQ_{X}\opt}\left[ \sum_{c=1}^C t_c\opt \mathbbm{1}_{\wh x_c}(X)\right] \label{eq:Q-opt-14}\\
    	&= \EE_{\QQ_{X}\opt}\left[   \EE_{\QQ_{Y|X}\opt} \left[ \ell_\lambda(X, Y, w) \right] \right] = \EE_{\QQ\opt} \Big[ \ell_\lambda(X, Y, w )\Big] \notag.
    	\end{align}
        where the set $\mc Q$ in~\eqref{eq:Q-opt-12} is defined as in~\eqref{eq:Q-def}. Equality~\eqref{eq:Q-opt-13} follows from Lemma~\ref{lemma:support-solution} and from the definition of $\alpha\opt$ and $\beta\opt$ that solve~\eqref{eq:Q-opt-1}-\eqref{eq:Q-opt-2}. Equality~\eqref{eq:Q-opt-14} follows from the definition of~$\QQ_X\opt$. The proof is completed.
\end{proof}

%%%%%%%%%%%%%%%%%%%%%
\section{Auxiliary Results}
\label{sec:app-aux}

\begin{lemma}[Locally strongly convex parameter] \label{lemma:strong-convexity}
    If $\Psi$ is locally strongly smooth, and at $\wh \theta$, the smoothness parameter is $\sigma$, then $\phi$ is locally strongly convex at $\wh \mu = \nabla \Psi(\wh \theta)$ with strongly convex parameter $1/\sigma$ in a sufficiently small neighbourhood of $\wh \mu$.
\end{lemma}

\begin{proof}[Proof of Lemma~\ref{lemma:strong-convexity}]
The proof follows directly from the proof of \cite[Theorem~4.1]{ref:goebel2008local}. By the definition of locally strongly smooth, for some $\Theta^{'}\subseteq\Theta$ neighborhood of $\wh\theta$, we have for $\theta\in\Theta^{'}$
\[
    \Psi(\theta)\le \Psi(\wh \theta)+\inner{\nabla\Psi(\wh \theta)}{\theta-\wh\theta}+\frac{\sigma}{2}\|\theta-\wh\theta\|_2^2.
\]
Since $\wh \mu=\nabla\Psi(\wh\theta)$ and $\phi(\wh\mu)=\inner{\wh\mu}{\wh\theta}-\Psi(\wh\theta)$, we have
\begin{align*}
    \phi(\mu)&=\sup_{\theta\in\Theta}\left(\inner{\mu}{\theta}-\Psi(\theta)\right)\\
    &\geq \sup_{\theta\in\Theta^{'}}\left(\inner{\mu}{\theta}-\Psi(\wh\theta)-\inner{\wh\mu}{\theta-\wh\theta}-\frac{\sigma}{2}\|\theta-\wh\theta\|_2^2\right)\\
    &=\inner{\wh\mu}{\wh\theta}-\Psi(\wh\theta)+\sup_{\theta\in\Theta^{'}}\left(\inner{\mu}{\theta}-\inner{\wh\mu}{\theta}-\frac{\sigma}{2}\|\theta-\wh\theta\|_2^2\right)\\
    &=\phi(\wh\mu) +\inner{\wh\theta}{\mu-\wh\mu}+\sup_{\theta\in\Theta^{'}}\left(\inner{\mu-\wh\mu}{\theta-\wh\theta}-\frac{\sigma}{2}\|\theta-\wh\theta\|_2^2\right).
\end{align*}
In the last step, note that $\wh\theta=\nabla \phi(\wh\mu)$. Taking $\theta-\wh\theta=\alpha(\mu-\wh\mu)$ where $\alpha=1/\sigma$. $\theta\in\Theta^{'}$ if $\mu-\wh \mu$ is sufficiently small. We have
\[
\sup_{\theta\in\Theta^{'}}\left(\inner{\mu-\wh\mu}{\theta-\wh\theta}-\frac{\sigma}{2}\|\theta-\wh\theta\|_2^2\right)\geq (\alpha-\frac{\sigma}{2}\alpha^2)\|\mu-\wh\mu\|_2^2=\frac{1}{2\sigma}\|\mu-\wh\mu\|_2^2.
\]
Therefore $\phi$ is locally strongly convex at $\wh \mu$ with strongly convex parameter $1/\sigma$.
\end{proof}
In Proposition~\ref{prop:surrogate}, since $\Psi$ is locally Lipschitz continuous, we have that $\Psi$ is locally strongly smooth with smoothness parameter $\sigma_c$ at $\wh\theta_c$, where $\sigma_c$ can be chosen as the local Lipschitz constant for a neighborhood around $\wh\theta_c$. By Lemma~\ref{lemma:strong-convexity} and the proof of Proposition~\ref{prop:surrogate}, for sufficiently small $\rho_c,c=1,\ldots,C$, we can choose $m$ explicitly as $m=\min_c1/\sigma_c$, thus $\kappa_2=\sqrt{2\max_c\rho_c\cdot\max_c\sigma_c}$.

% \begin{example}[Linear regression]
% The multi-variate response linear regression model with known covariance matrix $\Sigma$ is specified with $\nu$ being the Lebesgue measure on $\mc Y =\mathbb{R}^{d}$, $h(y)=\exp(-\frac{1}{2}y^\top \Sigma^{-1} y)/\sqrt{(2\pi)^{d}|\Sigma|^2}$, the sufficient statistic $T(y)=y$, the natural parameter space $\Theta=\mathbb{R}$ and the log-partition function $\Psi(\theta)=\frac{1}{2}\theta^\top\Sigma\theta$. Under a linear parameter mapping $\lambda:(w,x)\mapsto \Sigma^{-1}wx$, we have
% \[
%   Y | X = x \sim \mathrm{Normal}(w_{0}x,\Sigma), \qquad f(Y = y | X = x) = h(y)\exp(\inner{\lambda(w_{0},x)}{y}-\Psi(\lambda(w_{0},x))).
% \]
% \end{example}
% %\xuhui{Gamma and Negetive binomial both has $\Theta=$ half space.}

% \begin{example}[Negative binomial regression]
% The negative binomial regression model with known $r$ is specified with $\nu$ being a counting measure on $\mc Y = \mbb N$, $h(y)=\binom{y+r-1}{y}$, the sufficient statistic $T(y)=y$, the natural parameter space $\Theta = \mathbb{R}_{--}$ and the log-partition function
% $\Psi(\theta)=-r\log(1-\exp(\theta))$. Under a non-linear parameter mapping $\lambda:(w,x)\mapsto \log(\frac{\exp(w^\top x)}{r+\exp(w^\top x)})$, we have
% \[
%  Y | X = x \sim \mathrm{NB}(r,p=\frac{r}{r+\exp(w_0^\top x)}), \qquad P(Y=k|X=x)=\binom{k+r-1}{k}p^{r}(1-p)^{k}.
% \]
% \end{example}

%%%%%%%%%%%%%
\bibliographystyle{abbrv}
\bibliography{arxiv.bbl}

\end{document}